\documentclass{article}

\usepackage[preprint]{neurips_2024}

\usepackage[utf8]{inputenc} 
\usepackage[T1]{fontenc}    
\usepackage{hyperref}       
\usepackage{url}            
\usepackage{booktabs}       
\usepackage{amsfonts}       
\usepackage{nicefrac}       
\usepackage{microtype}      
\usepackage{xcolor}         

\def\ourAlg{\textsc{Malcom-PSGD}\xspace}

\usepackage{xspace}
\newcommand{\lnorm}{\left\|}
\newcommand{\rnorm}{\right\|}

\def\Yv{{\bm{Y}}}
\def\yv{{\bm{y}}}
\def\Xv{{\bm{X}}}
\def\xv{{\bm{x}}}
\def\Wv{{\bm{W}}}

\newcommand{\Iv}{{\bf I}}
\newcommand{\half}{\frac{1}{2}}
\newcommand{\Phiv}{\boldsymbol \Phi}
\newcommand{\phiv}{\boldsymbol \phi}
\newcommand{\norm}[1]{\left\lVert#1\right\rVert}
\newcommand{\paren}[1]{\left( #1 \right)}

\newcommand{\eqref}[1]{(\ref{#1})}
\newcommand\ie{i.e.,\xspace}
\def\bF{{\bm{F}}}

\usepackage[normalem]{ulem}

\usepackage{amsmath,amsfonts,bm}
\usepackage{amssymb}

\usepackage{amsthm}
\theoremstyle{plain}

\newtheorem{lemma}{Lemma}
\newtheorem{assumption}{Assumption}
\newtheorem{observation}{Observation}

\newtheorem{corollary}{Corollary}

\newtheorem{theorem}{Theorem}

\def\1{\bm{1}}

\def\vzero{{\bm{0}}}
\def\vone{{\bm{1}}}

\def\vg{{\bm{g}}}

\def\vp{{\bm{p}}}
\def\pv{{\bm{p}}}
\def\vq{{\bm{q}}}

\def\vu{{\bm{u}}}
\def\vv{{\bm{v}}}

\def\vx{{\bm{x}}}
\def\vy{{\bm{y}}}
\def\vz{{\bm{z}}}

\def\mA{{\bm{A}}}
\def\mB{{\bm{B}}}

\def\mI{{\bm{I}}}

\def\mQ{{\bm{Q}}}

\def\mW{{\bm{W}}}
\def\mX{{\bm{X}}}
\def\mY{{\bm{Y}}}
\def\mZ{{\bm{Z}}}

\DeclareMathAlphabet{\mathsfit}{\encodingdefault}{\sfdefault}{m}{sl}
\SetMathAlphabet{\mathsfit}{bold}{\encodingdefault}{\sfdefault}{bx}{n}

\newcommand{\E}{\mathbb{E}}

\newcommand{\R}{\mathbb{R}}

\newcommand{\Var}{\mathrm{Var}}

\DeclareMathOperator*{\argmin}{arg\,min}

\DeclareMathOperator{\sign}{sign}

\usepackage{amsmath}
\usepackage{amssymb}
\usepackage{mathtools}
\usepackage{amsthm}
\usepackage{enumitem}
\usepackage{bbm}
\usepackage{microtype}
\usepackage{graphicx}
\usepackage{subfigure}
\usepackage{booktabs} 
\usepackage{algorithm}
\usepackage[noend]{algpseudocode}
\usepackage{tabularx}
\usepackage{caption}
\usepackage{subcaption}
\usepackage{epstopdf}
\usepackage[export]{adjustbox}

\title{Compressed and Sparse Models for Non-Convex Decentralized Learning}

\author{%
  Andrew Campbell \\
 Electrical and Computer Engineering\\
  Cornell University\\
  Ithaca, NY 14850 \\
  \texttt{ac2458@cornell.edu}
  \And
  Hang Liu \\
  Electrical and Computer Engineering\\
  Cornell University\\
  Ithaca, NY 14850 \\
  \texttt{hl2382@cornell.edu}
  \AND
  Leah Woldemariam \\
   Electrical and Computer Engineering\\
  Cornell University\\
  Ithaca, NY 14850 \\
  \texttt{lsw85@cornell.edu}
  \And
  Anna Scaglione \\
  Electrical and Computer Engineering\\
  Cornell University\\
  Ithaca, NY 14850 \\
  \texttt{as337@cornell.edu}
}

\begin{document}

\maketitle

\begin{abstract}
  Recent research highlights frequent model communication as a significant bottleneck to the efficiency of decentralized machine learning (ML), especially for large-scale and over-parameterized neural networks (NNs). To address this, we present \ourAlg, a novel decentralized ML algorithm that combines gradient compression techniques with model sparsification. We promote model sparsity by adding $\ell_1$
 regularization to the objective and present a decentralized proximal SGD method for training. Our approach employs vector source coding and dithering-based quantization for the compressed gradient communication of sparsified models. Our analysis demonstrates that \ourAlg achieves a convergence rate of $\mathcal{O}(1/\sqrt{t})$
 with respect to the iterations $t$, assuming a constant consensus and learning rate. This result is supported by our proof for the convergence of non-convex compressed Proximal SGD methods. Additionally, we conduct a bit analysis, providing a closed-form expression for the communication costs associated with \ourAlg. Numerical results verify our theoretical findings and demonstrate that our method reduces communication costs by approximately 
75\% when compared to the state-of-the-art.
\end{abstract}

\section{Introduction}
\label{intro}
With the growing prevalence of computationally capable edge devices, there is a necessity for efficient ML algorithms that preserve data locality and privacy. One popular approach is decentralized ML. Under this regime, nodes within the network learn a global model by iterative model communication, whilst preserving data locality \citep{og_ddnn,fedavg,hsieh2020non}. 
It has been shown that decentralized ML algorithms can achieve similar accuracy and convergence rates to their centralized counterparts under certain network connectivity conditions \citep{nedic_multiAgent, dist_op_networks, choco-convex, Decent_geq_Cent}. 
While decentralized ML eliminates the need for data uploading, \citep{bonawitz2019towards, van2009multi, fed_survey} observed that iterative model communication over rate-constrained channels creates a bottleneck.
Particularly, the performance of decentralized training of large-scale deep neural network (DNN) models is limited by excessively high communication costs \citep{ProjectAdam, seide20141, Strom2015, xu2020compressed, kairouz2021advances}.

Contemporary research has shown that the efficiency of model communication can be enhanced by model sparsification and gradient/model compression \citep{NEURIPS2018_3328bdf9,renggli2019sparcml}. Model sparsification involves reducing the dimensionality by either pruning or regularizing the loss function, thereby inducing several zero coefficients in the resulting models \citep{sparseML}.
Conversely, gradient/model compression typically employs  methods such as quantization and source coding to compress the local model updates prior to communication \cite{qsgd, choco, choco-convex, han2015deep, wen2017terngrad, lin2018deep, seide20141, beer, other_grad_sparse, ye2018communication, tang2019doublesqueeze}. The latter approaches better reflect the imprecision that is associated with transmitting continuous data over a rate-limited environment. 

\textbf{Related Work}. Recently, \cite{choco-convex,pzhang17g_distributed,dist_op_networks} adopted the gossiping algorithm proposed by \cite{og_gossip} for decentralized ML with convex and smooth loss functions, where convergence is guaranteed with constant and diminishing step sizes.  
Continuing this work, \cite{choco,Decent_geq_Cent,nadiradze2021asynchronous,ecd_dcd} generalized the approach to smooth but non-convex loss functions and provided convergence guarantees. In contrast, \citep{mancino2023proximal} consider the proximal stochastic gradient descent (SGD) method, but require a variance reduction step for convergence. 
Finally, these SGD based methods were accelerated using stochastic gradient tracking (SGT) in \cite{sparsesgt, beer, yau2023docom,xin2021stochastic}. Notably, \cite{sparsesgt} proves convergence of the proximal SGT method in the decentralized non-convex and compressed setting. However, it remains an unresolved problem to determine the conditions under which proximal SGD is guaranteed to converge. In this work, as one of our contributions, we gain meaningful insights on the convergence of compressed proximal SGD. 

The above works also compress the model updates prior to communication and aggregation, but their analysis focuses on convergence and accuracy rather than on the achievable bit rate. \cite{qsgd} fills this gap by analyzing a dithering-based quantization scheme as well as a symbol-wise source coding method for model compression, termed QSGD. However, \cite{qsgd} did not explicitly explore the effects of model sparsity that are naturally induced during compressed decentralized training.
\cite{lap_residule} shows that sparsity is essential for NN compression and models can be approximated by a Laplacian distribution. Furthermore, the authors in \cite{grad_sparse, other_grad_sparse} analytically and numerically demonstrated the rate gains from gradient sparsification, but do not relate this sparsification to the source encoding scheme.
While \cite{qsgd} provides a bit analysis and \cite{lap_residule} analyzes the effect of sparsity, the literature lacks a symbiotic analysis that integrates sparsity with bit utilization. We fill this gap.  

\textbf{The Gap.} For the most part, the literature on decentralized learning has focused on decreasing the number of training iterations but has overlooked decreasing the bits per iteration. Within the compressed decentralized ML domain, the literature has focused on proving convergence given different variations of gradient methods with different assumptions on the objective function \citep{choco, sparsesgt,beer,ecd_dcd,NEURIPS2021_362c9930, vogels2020practical}. 
However, there is no analysis of the convergence rate of non-convex decentralized proximal SGD methods with quantized updates, 
nor has there been an analysis of the actual bit utilization induced by the code design, with the exception of \cite{qsgd}.  
Furthermore, the impact of sparsity, model aggregation schemes, and compression on bit utilization has not been studied. These gaps motivate our work.

\textbf{Our Contributions.} We introduce the \textbf{M}ulti-\textbf{A}gent \textbf{L}earning via \textbf{Com}pressed updates for \textbf{P}roximal \textbf{S}tochastic \textbf{G}radient \textbf{D}escent (\ourAlg) algorithm, which communicates compressed and coded model updates.
\ourAlg takes full advantage of the empirical distribution of the model weights that is induced by the $\ell_1$ regularization (not just the model sparsification) to reduce the model dimension during the SGD update.  
Our contributions are summarized as follows:

\begin{itemize}[align=parleft,nosep,leftmargin=*]\label{contri}
\item We introduce \ourAlg, which utilizes model sparsification via $\ell_1$ regularization and gradient compression through quantization followed by a vector source encoding scheme that approaches the entropy coding bound for model representation to reduce the communication costs for decentralized ML. Through analysis and numerical evaluation, we show that the amalgamation of these techniques enhances communication efficiency and boosts convergence. Numerically, we show a 75\% bit rate reduction when compared to the state of the art. 
     \item We prove that \ourAlg converges in the objective value for non-convex and non-smooth loss functions with a given compression error. We are the first to show that non-convex decentralized proximal SGD with compression errors converges, without utilizing variance reduction or gradient tracking methods. \ourAlg exhibits a convergence rate of $\mathcal{O}(1/\sqrt{t})$ and a consensus rate of $\mathcal{O}(1/t)$ with constant learning rates and consensus step sizes, where $t$ represents the number of training rounds, which are in the same order of  non-compressed proximal SGD convergence results.  
    \item We introduce a continuous adaptation of the discrete vector source encoding scheme from \cite{woldemariam2023low} by encoding the support vector of entries that match a certain quantization level and quantify the impact of the $\ell_1$ regularization of \ourAlg on the bit efficiency. We provide an explicit closed-form expression of the number of bits and show that \ourAlg has an asymptotic bit utilization rate of $\mathcal{O}(d \exp{(-1/L)})$ where $d$ is the dimensionality of the model and $L$ is the number of quantization levels (precision) of the compression scheme. This improves the state of the art whose compression rate grows at a rate of $\mathcal{O}(d \log(L))$.
\end{itemize}

\section{Learning Sparse Models}

\label{sec:algo} 
The decentralized ML task we consider is the minimization of the sum of local empirical loss functions for a network of NNs given by  
\begin{align}
	  \min_{\vx} \dfrac{1}{n}\sum_{i=1}^nF_i(\bm x;\mathcal{D}_i), ~ F_i(\vx):=\dfrac{1}{|D_i|}\sum^{|D_i|}_{j=1}f_i(\vx;\xi_{i,j}) \label{eq:obj_func_og},
\end{align}
where $\vx\in\R^{d}$ represents the model parameters, $F_i(\cdot;\mathcal{D}_i)$ is the local loss on node $i$ with respect to (w.r.t.) the local dataset $\mathcal{D}_i$, and $f_i(\cdot;\xi_{i,j})$ represents the loss w.r.t. the data sample $\xi_{i,j}$. We assume that $F_i$ is finite-sum, smooth, but in general non-convex, which represents a broad range of machine learning applications, such as logistic regression, support vector machine, deep neural networks (DNNs), etc. 

\textbf{Model Sparsification}. As shown in \citep{ProjectAdam, seide20141, Strom2015}, the frequent model communication is regarded as the bottleneck in the decentralized optimization of \eqref{eq:obj_func_og}, especially when optimizing large-scale ML models, such as DNNs. Model sparsity reduces the entropy and thus can be exploited to enhance the compression efficiency in local model communication.  The sparsity of the model parameter $\xv$ can be promoted by adding $\ell_1$ regularization to the objective. Specifically, we replace the original objective in \eqref{eq:obj_func_og} with the following problem:
\begin{align}
    \min_{\vx\in\mathbb{R}^d} \left\{ \mathcal{F}(\bm x):=\dfrac{1}{n}\sum_{i=1}^nF_i(\bm x;\mathcal{D}_i)+\mu\|\vx\|_1  \right\}  \label{eq:obj_func},
\end{align} 
where $\mu>0$ is a predefined penalty parameter. The $\ell_1$ regularization is an effective convex surrogate for the $\ell_0$-norm sparsity function, and thus is established as a computationally efficient approach for promoting sparsity in the fields of compressed sensing and sparse coding \citep{l1_good}. However, adding $\ell_1$ regularization makes the objective in \eqref{eq:obj_func} non-smooth.
\section{Proposed Decentralized Algorithm}\label{sec:alg_intro}
\ourAlg solves (\ref{eq:obj_func}) using decentralized proximal SGD
with five major steps: \emph{local SGD}, \emph{proximal optimization}, \emph{residual compression}, \emph{communication and encoding}, and \emph{consensus aggregation}. We refer to 
Algorithm \ref{alg:quant} for the illustration of the steps. Unless it is stated otherwise, we consider a synchronous decentralized learning framework over a static network. 
As mentioned before, decentralized optimization algorithms solve \eqref{eq:obj_func} by an alternation of local optimization and consensus aggregation. Specifically, defining $\xv_i$ as the local model parameters in node $i$ for $\forall 1\leq i\leq n$, \eqref{eq:obj_func} can be recast as 
\begin{align}
	&\min_{\vx_i\in\mathbb{R}^d,\forall i} ~~\dfrac{1}{n}\sum_{i=1}^n \left(F_i(\bm x_i;\mathcal{D}_i)+\mu\|\vx_i\|_1\right)\\
 &~~\mbox{subject to}~~\xv_i=\xv_j,\forall  (i,j)\in E. \nonumber
\end{align}   \label{eq:obj_func2}The nodes iteratively update the local solutions $\{\xv_i\}$ throughout $T$ iterations. Denote the $i$-th local model in iteration $t$ as $\xv^{(t)}_i$.  In iteration $t+1$, each node $i$ first updates a local solution, denoted by $\vz^{(t)}_i$, by first minimizing  $\mathcal{F}_i(\bm x_i)$ based on the preceding solution $\xv^{(t)}_i$, generating $\vx_i^{(t+1/2)}$, followed by proximal optimization generating $\vz_i^{(t)}$. Once every node completes its local update, the nodes then communicate with their neighbors for consensus aggregation.
The unique and critical steps of \ourAlg are below.
\begin{algorithm}[!t]
  \caption{\ourAlg 
    \label{alg:quant}}
  \begin{algorithmic}[1]
    \State {\bfseries Initialize:\nonumber} $\vx_i^{(0)},\vy^{(-1)}_i=\vzero \in \R^d$,$\forall 1\leq i\leq n$.
    \For{$t\in [0,\hdots, T-1]$} 
    \Comment{All nodes $i$ do in parallel}
    \State{$\vx^{(t+1/2)}_i=\vx_i^{(t)}-\eta_t\nabla F_i(\vx_i^{(t)},{\bf \xi}^{(t)}_i)$}
    \State{$ \vz^{(t)}_i=\text{prox}_{\eta_t,\mu\lnorm\cdot\rnorm}\left(\vx_i^{(t+1/2)}\right)$}
    \Comment{Proximal optimization by soft-thresholding; see \eqref{eq:soft_thres}.}
    \State{$\vq_i^{(t)}=Q(\vz^{(t}_i-\vy^{(t-1)}_i)$} 
    \Comment{Residual quantization; see \eqref{eq:q(x)}.}
    \For{$j\in\mathcal{N}_i$}
    \State{\textbf{Encode and send} $\vq_i^{(t)}$} 
    \Comment{Residual communication with source encoding from Alg. \ref{alg:encode}.}
    \State{\textbf{Receive and decode} $\vq_j^{(t)}$} 
    \Comment{Receiver decoding.}
    \State{$\vy^{(t)}_j=\vq^{(t)}_j+\vy_j^{(t-1)}$}
    \EndFor
    \State{$\vx_i^{(t+1)}=\vz^{(t)}_i+\gamma\sum_{j\neq i}w_{i,j}\left(\vy_j^{(t)}-\vy^{(t)}_i\right)$}
    \Comment{Consensus aggregation; cf. (\ref{eq:conse}).}
    
    \EndFor
  \end{algorithmic}
\end{algorithm}

\textbf{Proximal Optimization.} To tackle the non-smoothness of the objective function in (\ref{eq:obj_func}), we adopt the proximal SGD method, which decomposes (\ref{eq:obj_func}) into a smooth but non-convex component $F_i(\vx)$ and a convex but non-smooth component $\mu\|\vx\|_1$. 
$\vx$ is updated by the SGD method as previously described and is subsequently fed to the proximal operation (w.r.t. $\mu\|\vx\|_1$) generating $\vz^{(t)}_i$. Formally, \begin{align}
\vz_i^{(t+1)}&=\text{prox}_{\eta,\mu\lnorm\cdot\rnorm}\left(\vx_i^{(t+1/2)}\right) =\argmin_{\vu}\left\{\mu\eta\norm{\vu}_1+{\dfrac{1}{2}\norm{\vu-\vx_i^{(t+1/2)}}^2}\right\} \label{eq:soft_thres},
\end{align}
This operation is characterized by a closed-form update expression derived from the soft-thresholding function. Specifically we have
\begin{align}
    \text{prox}_{\eta_t,\mu\lnorm\cdot\rnorm}\left(\vz_i^{(t+1)}\right)=\max\left\{(|\vz_i^{(t+1)}|-\mu\eta),0\right\}\sign(\vx_i^{(t+1/2)}).
\end{align} The soft-thresholding operation promotes model sparsity by truncating values with a magnitude less than $\mu\eta$. This step is important in accelerating convergence and conserving communication bandwidth as the encoding scheme directly exploits this sparsity.

\textbf{Residual Quantization.} Let $\vp^{(t)}\triangleq\vz^{(t)}_i-\yv^{(t-1)}_i$ denote the residual given in Step 5 of Algorithm \ref{alg:quant}. For the sake of exposition, we drop the time indexing. We adapt the uniform quantization scheme from $\textsc{QSGD}$ with uniform dithering and adaptive range. Let $\hat{p}_i=\dfrac{p_i-\min(\vp)}{ \max(\vp) -\min(\vp)}$ where $p_i$ is the $i$-th entry of $\vp$. We set the quantization function (step 5 of Algorithm \ref{alg:quant}) as
\begin{align}
    Q(p_i)&=\underset{p_i \neq 0}{\vone}\dfrac{L^2}{L^2+d}\left(\frac{1}{L}\lfloor \hat{p}_iL+u\rfloor\right),
    \label{eq:q(x)}
\end{align} 
where $\vone_A$ is the indicator function and $u$ is drawn from the uniform distribution over $[0,1]$. Here, the input values are adaptively normalized into the range $[0,1]$ by the extreme values of the input vector $\vp$. For a fixed $L$, this adaptive normalization increases the precision as the updates range diminishes. The choice of the quantizer leads to the following lemma. Here, we extend the notation by using $Q(\vp)$ to also denote the quantization output after de-normalization at the receiver.
\begin{lemma}
    For any input $\vp$, the received quantized result after the de-normalization
\label{lem:quant}$Q(\pv):\R^d\rightarrow\R^{d}$ satisfies that \begin{align}
			\mathbb{E}[\lVert{Q({\bf p})-\bf p}\rVert^2]\leq (1-\frac{1}{\tau})\lVert{\bf p}\rVert^2, \qquad Q({\bf 0})={\bf 0},\label{eq:Qnorm}
		\end{align}
  where $\tau\in (0,1)$ is a constant measuring the compression error bound. Specifically, we have $\tau=1+d/L^2$ for the quantizer in \eqref{eq:q(x)}.
\end{lemma} The proof can be found in Appendix \ref{proof:quant}.
Lemma \ref{lem:quant} is equivalent to the assumptions about compression operators studied in \cite{choco,choco-convex,beer,sparsesgt}. Not only does this make our compression scheme interoperable with the state of the art, but our convergence analysis (Theorem \ref{th:converge} as shown in the next section) can be readily extended to other popular compression operators including $\textsc{QSGD}$ \citep{qsgd}, $\text{top}_k$, $\text{rand}_k$ \citep{other_grad_sparse}, and rescaled unbiased estimators.

For details on the encoding scheme (step 7-8 of Algorithm \ref{alg:quant}) we refer to Section \ref{sec:bit_analysis}.

\textbf{Communication}. At each iteration of the algorithm, nodes individually update local models using local SGD with a mini-batch of their own data sets, followed by the proximal operation. Then, they communicate the update with their neighbors according to a network topology specified by an undirected graph $\mathcal{G}(\{1,\cdots,n\},E)$, where $E$ denotes the edge set representing existing communication links among the nodes. 
Let $\mathcal{N}_i:=\{j\in V : (i,j)\in E  \}$ denote the neighborhood of $i$. We define a mixing matrix $\mW\in\R^{n\times n}$, with the $(i,j)-$th entry $w_{ij}$ denoting the weight of the edge between $i$ and $j$. We make the following standard assumptions about $\mW$:
\begin{assumption}[The mixing matrix]\label{as:mixMatrix}The mixing matrix $\mW$ satisfies the following conditions: 
 \begin{itemize}[nosep,leftmargin=0.2 in]
     \item[i.] $w_{ij}>0$ if and only if there exists an edge between nodes $i$ and $j$.
    \item[ii.] The underlying graph is undirected, implying that $\mW$ is symmetric.
     \item[iii.] The rows and columns of $\mW$ have a sum equal to one, i.e., it is doubly-stochastic.
     \item[iv.] $\mathcal{G}(\{1,\cdots,n\},E)$ is strongly connected, implying that $\mW$ is irreducible.
\end{itemize}
\end{assumption}
We align the eigenvalues of $\mW$ in descending order of magnitude as
		$|\lambda_1|=1>|\lambda_2|\geq \cdots \geq |\lambda_n|.$
  
\textbf{Consensus Aggregation.} 
To achieve consensus we utilize the following standard aggregation scheme (see e.g. \citep{choco-convex,choco,beer,sparsesgt,vogels2020practical}),
\begin{align}
    \label{eq:conse} 
    \vx_i^{(t+1)}=\vz^{(t)}_i+\gamma\sum_{j\neq i}w_{i,j}\left(\vy_j^{(t)}-\vy^{(t)}_i\right).
\end{align}  
This scheme, critically, preserves the average of the iterates. That is, for the network average $\bar{\vx}^{(t+1)}:=\frac{1}{n}\sum_{i=1}^n\vx_i^{(t+1)}$, the errors due to aggregation are zero.    

\section{Convergence Analysis}
\label{sec:analysis}

In this section, we analyze the convergence conditions of the objective value of the proposed solution to \eqref{eq:obj_func}. We denote the average of local models in round $t$ by $\bar \xv^{(t)}:=\frac{1}{n}\sum_{i=1}^n \xv_i^{(t)}$. For ease of notation, we denote the model parameters in the matrix form by stacking the local models by column as 
$\mX^{(t)}:=[\vx_1^{(t)},\cdots,\vx_n^{(t)}]$ and $\overline\mX^{(t)}:=[\overline\vx^{(t)},\cdots,\overline\vx^{(t)}]$.
We impose the following assumptions on the training loss function, which is standard in the stochastic optimization literature.

\begin{assumption} \label{as:nice_f} Each local empirical loss function, i.e., $F_i$ in (\ref{eq:obj_func}), satisfies the following conditions.
\begin{enumerate}[nosep, leftmargin=0.15 in]
    \item[i.] \label{as:smoothFi} Each $\vx_i\mapsto F_i$ is Lipschitz smooth with constant $K_i$. As a result, the sum $\sum_i F_i$ is Lipschitz smooth with constant $K=\max_i K_i$.
    \item[ii.] \label{as:coercive} Each $F_i(\vx_i)+\mu\|\vx_i\|_1$ is proper, lower semi-continuous, and bounded below.
    \item[iii.] \label{as:F_unBiased}

All the full batch gradient vectors are bounded above by $G<\infty$, i.e.,
				$\|\nabla F_i\l(\vx_i^{(t)})\|\leq G ,\forall{i,t},$
        Moreover, the mini-batch stochastic gradient vectors are unbiased with bounded variances, i.e.,
\begin{align}\label{eq_gradbound}
&\E\left[\nabla F_i\left(\xv_i^{(t)};\xi_i^{(t)}\right)\right]=\nabla F_i\left(\xv_i^{(t)}\right),\nonumber\\
&\E[\|\nabla F_i(\vx,\xi_i)-\nabla F_i(\vx)\|^2]\leq \sigma_i^2,\forall i,t.
\end{align}	
\end{enumerate}
\end{assumption}
By the AM-GM inequality, (\ref{eq_gradbound}) implies 
 \begin{align}
    \sum_{i=1}^n&\E[\|\nabla F_i(\vx,\xi_i)\|^2]\leq 2 \sum_{i=1}^n \left(\E[\|\nabla F_i(\vx,\xi_i)\!-\!\nabla F_i(\vx)\|^2]+\|\nabla F_i(\vx)\|^2\right) \leq 2n(G^2+\sigma^2),\label{eq:var_bound}
\end{align}
where $\sigma^2\triangleq \frac{1}{n}\sum_{i=1}^n \sigma_i^2$ is the average gradient variance. In \eqref{eq:var_bound}, the term $\sigma^2$ measures the inexactness introduced to the mini-batch SGD step.

 To present the convergence result of \ourAlg, we define $\mathcal{F}(\mX)=1/n\sum_{i=1}^n(F_i(\vx_i)+\mu\norm{\vx_i}_1$) as the objective value of the regularized loss. Furthermore we introduce the generalized projected gradient of $\mathcal{F}$ as
 \begin{align}
     \frac{1}{n}\sum_{i=1}^n\vg_i(\vx^{(t)}_i)=\frac{1}{n}\sum_{i=1}^n\frac{1}{\eta}\left(\vx^{(t)}_i-\vz^{(t)}_i\right),
 \end{align} where $\tilde{\vg}(\cdot)$ refers to the stochastic projected gradient. For more details on the projected gradient we refer to Appendix \ref{app:proof_converge}. Using this definition we have the following Theorem.
 \begin{theorem}\label{th:converge} Suppose \ourAlg satisfies Assumptions \ref{as:mixMatrix} and \ref{as:F_unBiased}, there exists a constant consensus step size $\gamma$ and constant learning rate $\eta$, such that the projected gradient of \ourAlg diminishes by:
 \begin{align}
     \dfrac{1}{T+1}\sum_{t=0}^{T}\norm{\vg\left(\bar{\vx}^{(t)}\right)}^2&\leq 2\sqrt{\left(\mathcal{F}\left(\bar{\vx}^{(0)}\right)-\mathcal{F}^*\right)\dfrac{K\left(\sigma^2+4\mu^2d+G\sqrt{2G^2+2\sigma^2+\mu^2d}\right)}{(T+1)n\omega}}\nonumber\\
    &~~~~+7\left(\dfrac{K\sqrt{2G^2+2\sigma^2+\mu^2d}\left(\mathcal{F}\left(\bar{\vx}^{(0)}\right)-\mathcal{F}^*\right)}{(T+1)\omega}\right)^{\frac{2}{3}}\nonumber \\
    &~~~~+\dfrac{16K\left(\mathcal{F}\left(\bar{\vx}^{(0)}\right)-\mathcal{F}^*\right)}{T+1},
 \end{align}
 where $\omega=\frac{(1-|\lambda_2|)^2}{82\tau}$ with $\tau$ from Lemma \ref{lem:quant} and $\lambda_2$ from Assumption \ref{as:mixMatrix}. Therefore, \ourAlg has an asymptotic convergence rate of $\mathcal{O}(1/\sqrt{nT})$.
\end{theorem}

The proof of Theorem \ref{th:converge} can be found in Appendix \ref{app:proof_converge}. We are the first to show that non-convex, compressed, and decentralized Proximal-SGD converges without requiring variance reduction or gradient tracking techniques \citep{sparsesgt,beer,mancino2023proximal,yau2023docom,xin2021stochastic}. Interestingly, we match the asymptotic convergence rate of \textsc{Choco-SGD}. Before moving to the bit-analysis we have the following remarks.

\textbf{Remark.} Theorem \ref{th:converge} is for a constant consensus and constant learning rate, but the proof can readily be extended to the diminishing step sizes regime.

\textbf{Remark.} In general the proximal operation is not computationally simple. However, since we are solving the proximal operation w.r.t the $\ell_1$ norm we only occur a linear computational complexity. Additionally, Theorem \ref{th:converge} demonstrates a linear speed up w.r.t the number of workers. These two facts combined indicate that \ourAlg is a competitive choice in the decentralized bandwidth constrained regime. 


\section{Communication Bit Rate Analysis}
\label{sec:bit_analysis}

\noindent \textbf{Source Coding Scheme and its Code Rate}. The quantization scheme constrains the high-dimensional residual vectors to a limited set of values.  Model residuals tend to shrink due to algorithm convergence. In addition, the proximal operation promotes sparsity in the local models, further encouraging shrinkage in the quantized model residuals. This suggests that models are compressible, since high values will become increasingly rare. Expecting a model whose empirical distribution is low entropy, we employ the source coding scheme proposed in \cite{woldemariam2023low}, where the details are discussed in Algorithm \ref{alg:encode} of Appendix \ref{app:encode}. The scheme is inspired by the notion of encoding the support of a sparse input and the support of each quantization level individually in a manner that approaches the entropy coding bound. The basic idea is to encode the frequencies and positions of values over the support of each ${\bf q}_i^{(t)}$ by using Elias coding and Golomb coding, respectively.
 If the quantized coefficients tend to be concentrated around a limited number of modes, 
we can encode the support of those coefficients for each quantization level efficiently as we expect most levels to appear at a lower frequency.

For notational simplicity, we omit the indices $i$ and $t$ from the quantized encoding input vector ${\bf q}_i^{(t)}$, representing it as simply $\bf q$. The input vector $\bf q$ is encoded in two different structures, the type vector $\bm t(\bf q)$ and the support vectors $\{ s_{\ell}[\mathcal{I}_{\ell}] \}_{\ell=0}^{{L}-1}$. The type vector $\bm t(\bf q)$ storing the frequency information of each level, denoted by $\chi_{\ell},0\leq \ell\leq L-1$, is first created, where the $\ell$-th type is $t_{\ell}({\bf q})=\sum_{j=1}^d \delta(q_j-\chi_{\ell})$.  
Each entry of the type vector has an associated support vector $s_{\ell}$ denoting the positions in $\bf q$ that are equal to $\chi_{\ell}$, i.e. $s_{\ell}[j] = 1$ if $q_j = \chi_{\ell}$ and $s_{\ell}[j] = 0$ otherwise. 
Because $\sum_{\ell=0}^{L-1} s_{\ell}[j] = 1,\forall 1\leq j\leq d$, for a support vector $s_{\ell}$, 
all subsequent support vectors for levels $\ell' > \ell$ have $s_{\ell'}[j] = 0$. Thus, these subsequent support vectors do not need to encode positional information for the values $\mathcal{I}_{\ell}$, the set of indices where $s_{\ell}[i] = 1$. 
Let $s_{\ell}[\mathcal{I}_{\ell}]$ denote the indices within $s_{\ell}$ that are communicated. It follows from \cite{woldemariam2023low} that $\mathcal{I}_{\ell} = \mathcal{I}_0 \setminus \bigcup_{\ell' < \ell} \mathcal{I}_{\ell'}$ where $\mathcal{I}_0 = \{1, \dots, d\}$. The run-lengths within the support vectors $s_{\ell}[\mathcal{I}_{\ell}]$ are then encoded with Golomb encoding\footnote{For further details see Appendix \ref{app:encode}.}. 
We demonstrate, via Theorem \ref{th:bits}, that this encoding scheme combined with our sparse inputs yields a superior bit utilization rate when compared to the state of the art. This is largely because QSGD does not similarly utilize sparsity when designing their compression scheme.
As the training converges, we expect a diminishing residual to be quantized in Step 4 of Alg. \ref{alg:quant}, resulting in a sparse quantized vector.

Let $\phi^t(\cdot)$ be the probability mass function (PMF) of the quantization output ${\bf q}_i^{(t)}$ at iteration $t$. It follows that the type vector tends to $d\phi^t(\cdot)$ for $d \to \infty$, i.e.  $\frac{1}{d} {\bm t(\bf q)} \to \phi^t(\cdot)$.  Let $\phi^t(\cdot)$ be the permutation of $f^t(\cdot)$ that ensures $f^t(\cdot)$ is in nonincreasing order. The re-ordered frequency  $f^t(\cdot)$ satisfies $f_{\ell}^t \geq f_{\ell+1}^t,\forall \ell$ for the $\ell$-th quantization level with $0\leq\ell\leq L-1$.
The quantization mapping 
described in Section \ref{sec:alg_intro} scales an input vector according to its range, essentially shrinking the support of $f^t(\cdot)$ as $t$ grows, given that the precision is fixed. %
In equation (11) of \cite{woldemariam2023low}, it is shown that the bit length of encoding ${\bf q}_i^{(t)}$ is upper bounded by
\vspace{-0.2cm}
\begin{align}\label{eq8}
   d\cdot R(L) \leq d\Bigg(H(f^t) + 2.914(1 - f_0^t) + f_0^t\log_2f_0^t +\sum_{\ell=1}^{L-1}f_{\ell}^t\log_2(1 - \sum_{m=0}^{\ell -1}f_m^t)\Bigg),
\end{align}
where $f_0$ is the PMF associated with the most frequent quantization level, $H(f^t)$ is the entropy of $f^t$, $d$ is the size of the model, and $R(L)$ is the average bit rate. Recall that the entropy of $f^t$ is defined as $-\sum_{\ell=0}^{L -1} f^t_{\ell} \log_2 f^t_{\ell}$.
By computing the empirical PMF $f^t$ in each iteration, (\ref{eq8}) provides a formula for the required number of bits in each local model communication. As large values rarely occur due to the regularization, this method is particularly effective in harnessing compression gains. 
 
As in \cite{woldemariam2023low}, encoding the run-lengths with Golomb coding has a complexity of $\mathcal{O}(\log L)$, so there is a total complexity of $\mathcal{O}(d\log L)$. For comparison, \cite{qsgd} also has a total complexity of $\mathcal{O}(d \log L)$.

\textbf{Analysis Under Laplace Residual Model}.
The computation of communication costs using \eqref{eq8} necessitates an understanding of the prior distribution of the quantization inputs, \ie the local model residuals $\vp^{(t)} = \{\vx^{(t+\half)}_i-\yv^{(t-1)}_i\}$, which
is often unavailable prior to the training process. 
However, the $\ell_1$-regularization in \eqref{eq:obj_func} promotes sparse local models, and \cite{lap_residule} reported that the residual can be modeled by an i.i.d. zero-mean Laplace distribution characterized by the time-varying diversity parameter $\rho^{(t)}$. The sparser the model, the smaller $\rho^{(t)}$ is. Below is the main result of this analysis. 
%
\begin{theorem}
\label{th:bits}
Suppose \ourAlg's residuals $\vp^{(t)}$ follow a zero-mean Laplace distribution with parameter $\rho^{(t)}$. Let the range of the quantizer adapt to the mean of the distribution as $r^{(t)}\propto \rho^{(t)} 2(\ln(d)-\ln(\epsilon))$, since with probability $1-\epsilon$ for small $\epsilon\in[0,1]$ the adaptive range $\max(\bm p^{(t)})-\min(\bm p^{(t)})$ is bounded by this value. 
Then, the number of bits, $R(L)$, used in one round of model communication for any pair of nodes at iteration $t$ is such that:
\begin{align}
R(L) \leq& 2.914\exp{\left(-\frac{r^{(t)}}{\rho^{(t)}2L}\right)}\label{eq:final_bits}+2\sinh\left(\frac{r^{(t)}}{\rho^{(t)}2L}\right)\frac{1-\log \left(1-\exp{\left(-\frac{r^{(t)}}{\rho^{(t)}L}\right)}\right)}{\exp{\left(\frac{r^{(t)}}{\rho^{(t)} L}-1\right)}}.
\end{align} Since $r^{(t)}\propto 2(\ln(d)-\ln(\epsilon))$ for large $d$ and $L$, \eqref{eq:final_bits} implies that the bit rate grows at a rate of 
\vspace{-0.1cm}
\begin{align}
    \mathcal{O}\left(d\exp{\left(-\frac{1}{L}\right)}\right) \label{eq:complex_bit}.
\end{align}
\vspace{-0.2cm}
\end{theorem}
For the proof we refer to Appendix \ref{proof:bits_num} 
Theorem \ref{th:bits} suggests a significant improvement over the state of the art $\mathcal{O}(d\log L)$ \cite{qsgd}. Not only does $\mathcal{O}(d\exp({-1/L}))$ grow at a slower rate, it approaches a constant factor. This implies that as we increase the precision, our bit utilization approaches some constant. Furthermore, because of the adaptive range, after enough iterations have passed \ourAlg begins using a nearly constant number of bits per iteration. Finally, we note that the advantage obtained from decreasing $L$ must be balanced against the degradation in convergence speed because a smaller $L$ also results in higher quantization error and thus slower convergence, as shown in Section \ref{sec:analysis}.

\section{Numerical Results}
\label{sec:numerical}

In this section, we evaluate the performance of \ourAlg through simulations of decentralized learning tasks on image classification using the MNIST \citep{mnist} and CIFAR10 \citep{cifar} datasets. We compare \ourAlg with \textsc{Choco-SGD} \citep{choco} (paired with the compression scheme of \textsc{QSGD} given in \citep{qsgd}), \textsc{Choco-SGD}, paired with the compression scheme utilized by \ourAlg, and CDProxSGT \citep{sparsesgt}, paired with the \ourAlg compression scheme. We tested \textsc{Choco-SGD} with both our compression and \textsc{QSGD} so that the effect of the $\ell_1$ norm can be analyzed independently. For each of these compression schemes we fixed the precision level, i.e., the number of quantization levels over the interval $[0,1]$.   
We run training over a fully connected (FC.) and ``ring-like"\footnote{For details of the ``ring" like topology see Appendix \ref{app:add_exp} Fig. \ref{fig:ringtop}} network topology with 10 nodes each having an i.i.d distribution of data. ResNet18 \citep{he2016deep} was trained on CIFAR10 while a three-layer fully connected NN was trained on MNIST. The details of the experimental setup (including the hyper-parameters) can be found in Appendix \ref{app:add_exp}.   
\begin{table}[ht!]
\centering
    \begin{tabularx}{1.0\textwidth}{*{2}{>{\centering\arraybackslash}X}}
    \begin{tabular}{l|cccc}
        \hline
         & { \tiny CIFAR-Ring} & {\tiny MNIST-Ring}& {\tiny CIFAR-FC.}& {\tiny MNIST-FC.}\\
         \hline
        {\small Algorithms}  & \multicolumn{4}{c}{{\small Final Cumulative  bits (GB)}} \\ 
        \hline
        {\small Malcom-PSGD}  & {\footnotesize \textbf{589.36}} & {\footnotesize \textbf{626.52}} & {\footnotesize \textbf{481.98}}& {\footnotesize \textbf{144.34}}\\
    
        {\small Error-Free}   & {\footnotesize 10799.82} & {\footnotesize 13417.97} & {\footnotesize 8945.31} & {\footnotesize 1050.31}\\
    
        {\small Choco+QSGD}  & {\footnotesize 3174.91} & {\footnotesize 3172.22} & {\footnotesize 2107.91} & {\footnotesize 573.08}\\
    
        {\small Choco+OurComp}  & {\footnotesize 707.58} & {\footnotesize 796.30} & {\footnotesize 522.62} & {\footnotesize 149.62}\\
    
        {\small CDProxSGT}  & {\footnotesize 988.87} & {\footnotesize 857.94} & {\footnotesize 811.42} & {\footnotesize 283.75}\\
        \hline
        \end{tabular}
      & \hspace{2.2cm}\includegraphics[width=.3\textwidth, valign=c]{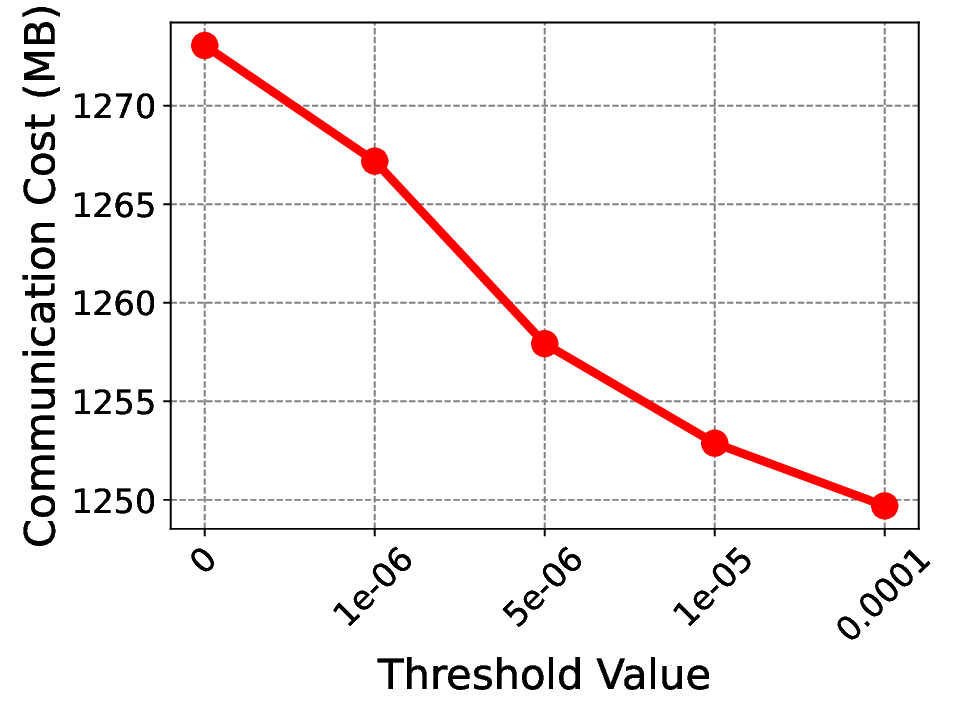}\label{fig:thres} \\
      \caption{Contains the cumulative number of bits that each algorithm used up to the cut-off as described in Section \ref{sec:numerical}.}\label{fig:table}
      & \captionof{figure}{\ourAlg was validated on the CIFAR10 dataset using the ring topology with varying threshold parameters.}
      \end{tabularx}
\end{table}

Figure \ref{fig:ring} and Figure \ref{fig:fully_con} showcase the testing accuracy and training loss of the different algorithms for the different datasets over the ring and fully connected topologies respectively. Notice in each of the accuracy plots there is a dashed horizontal line, denoting the accuracy cut-off, and each curve has a marker corresponding to when that curve crosses the accuracy cut-off, or the final point if it failed to cross the cut-off. The cumulative number of bits used to reach these points is reported in Table \ref{fig:table}. Figure \ref{fig:table} shows the cumulative number of bits utilized by \ourAlg over varying threshold ($\mu$) values. 

\textbf{\ourAlg vs. \textsc{Choco-SGD}.} We consider \textsc{Choco-SGD} combined with QSGD to be the state of the art. \ourAlg sees over a 75\% improvement in bit utilization when compared to \textsc{Choco-SGD}+QSGD. This verifies the analytical claim in Section \ref{sec:bit_analysis} that our compression scheme has a more efficient bit utilization. Furthermore, while both \textsc{Choco-SGD} and \ourAlg have similar accuracy versus iterations curves, \ourAlg tended to have a slightly faster convergence. This is likely due to the fact that sparse models are more stable under the quantization scheme and thus converge faster.      

\textbf{CDProxSGT vs. \ourAlg.} As shown in \cite{sparsesgt,NEURIPS2021_5f25fbe1,pu2021distributed} SGT-based methods require less training and communication iterations than the typical SGD methods. However, the accelerated convergence does not necessarily result in a total bandwidth or bit reduction as demonstrated in Table \ref{fig:table}. CDProxSGT often converges 4000 iterations before \ourAlg, but uses more bits. CDProxSGT, utilizing SGT, requires 2 communication steps at every iteration. This at minimum doubles the number of bits required. Finally, CDProxSGT compresses both the model residuals and the ``tracking" residual. While we have shown in Section \ref{sec:bit_analysis} that the compression scheme is advantageous for model residuals, there is currently no theory on characterizing the evolution of the ``tracking" residual.          
\begin{figure}
    \centering
    \begin{subfigure}
        \centering
        \includegraphics[width=.23\textwidth]{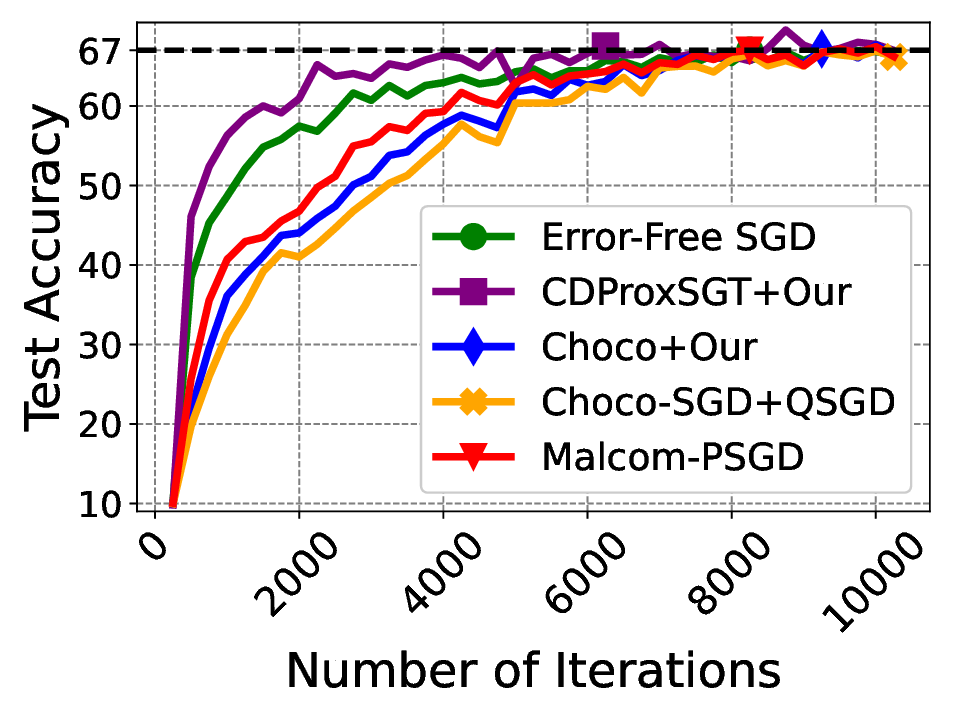}
    \end{subfigure}
    \begin{subfigure}
        \centering
        \includegraphics[width=.23\textwidth]{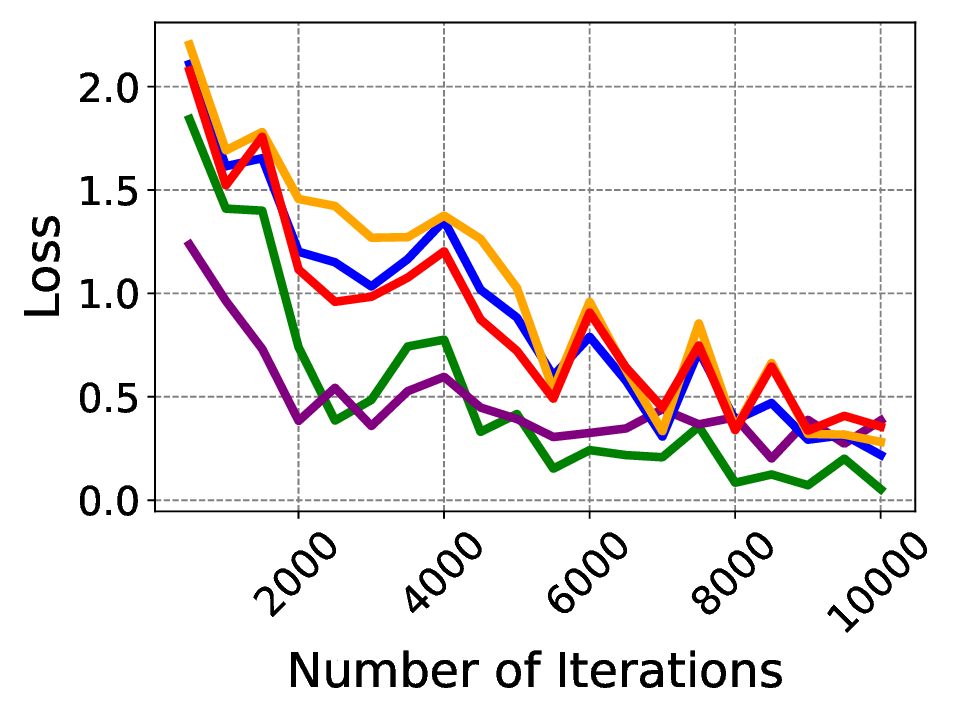}
    \end{subfigure}
        \begin{subfigure}
        \centering
    \end{subfigure}
    \begin{subfigure}
        \centering
        \includegraphics[width=.23\textwidth]{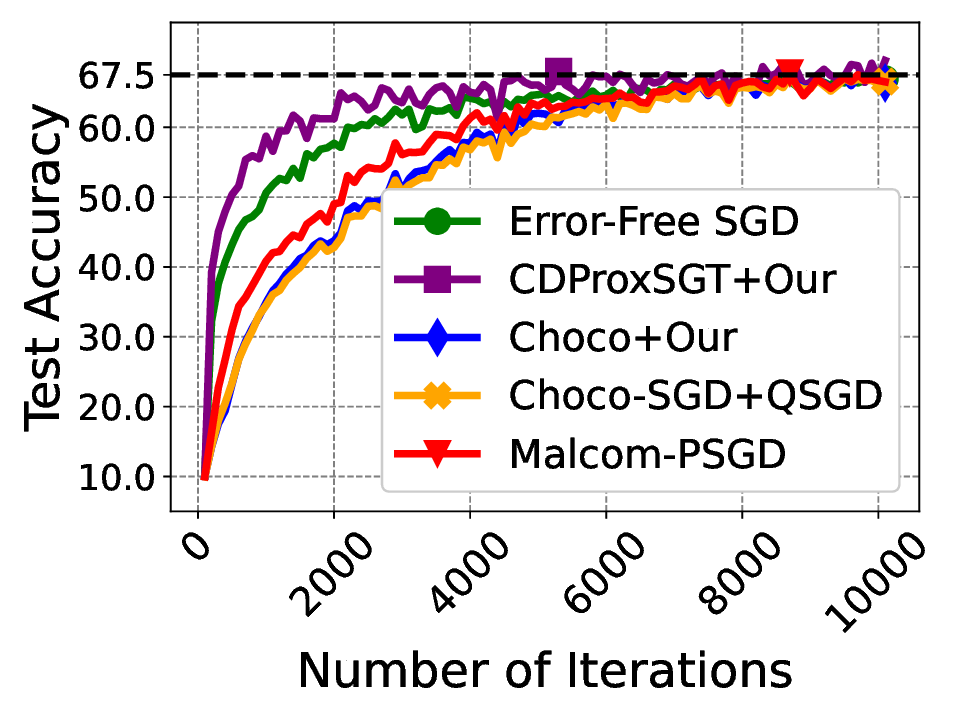}
    \end{subfigure}
    \begin{subfigure}
        \centering
        \includegraphics[width=.23\textwidth]{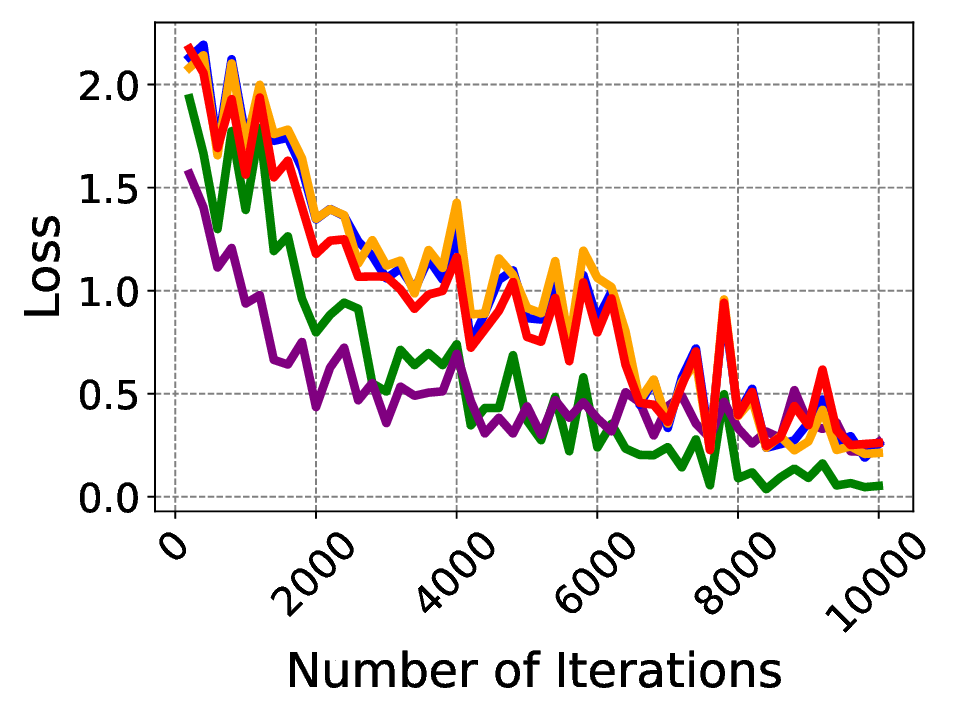}
    \end{subfigure}
    \caption{All plots are over the ``ring" topology. The top row corresponds to the ResNet setup over the CIFAR10 dataset while the bottom corresponds to the 3-layer FC. NN over MNIST. The left column contains the accuracy plots while the right contains the loss plots. The horizontal black dashed line refers to the accuracy cut off and the diamond-colored points are where the bits were sampled from. For the CIFAR10 setup there is an accuracy cut-off of 67.0 while the MNIST setup has a cut-off at 67.5.}
    \label{fig:ring}

\end{figure}

\begin{figure}
    
    \centering
    \begin{subfigure}
        \centering
        \includegraphics[width=.23\textwidth]{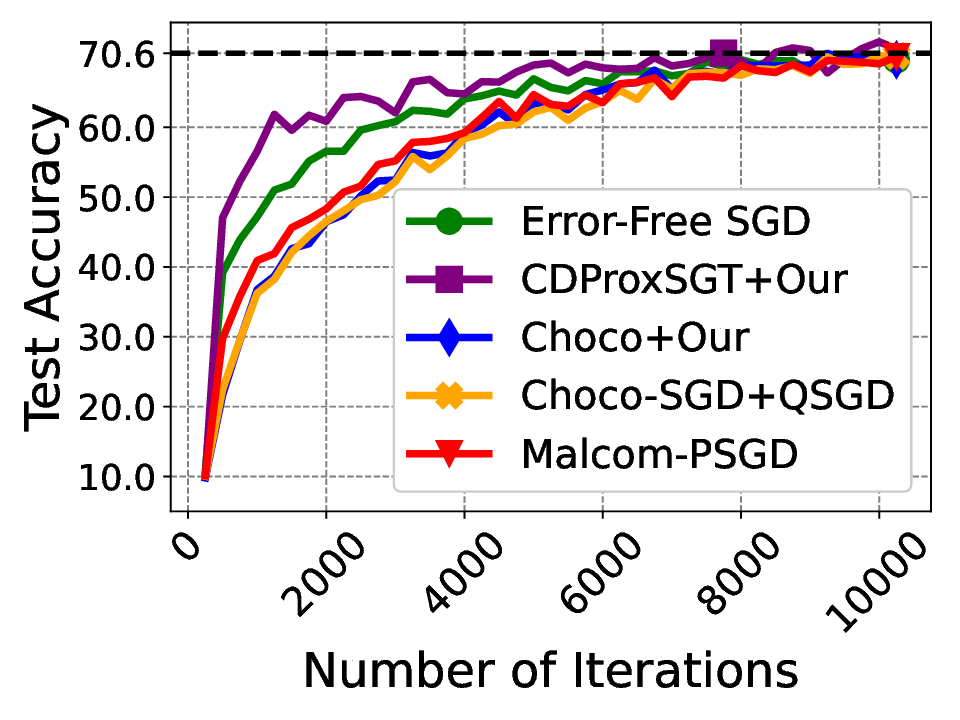}
    \end{subfigure}
    \begin{subfigure}
        \centering
        \includegraphics[width=.23\textwidth]{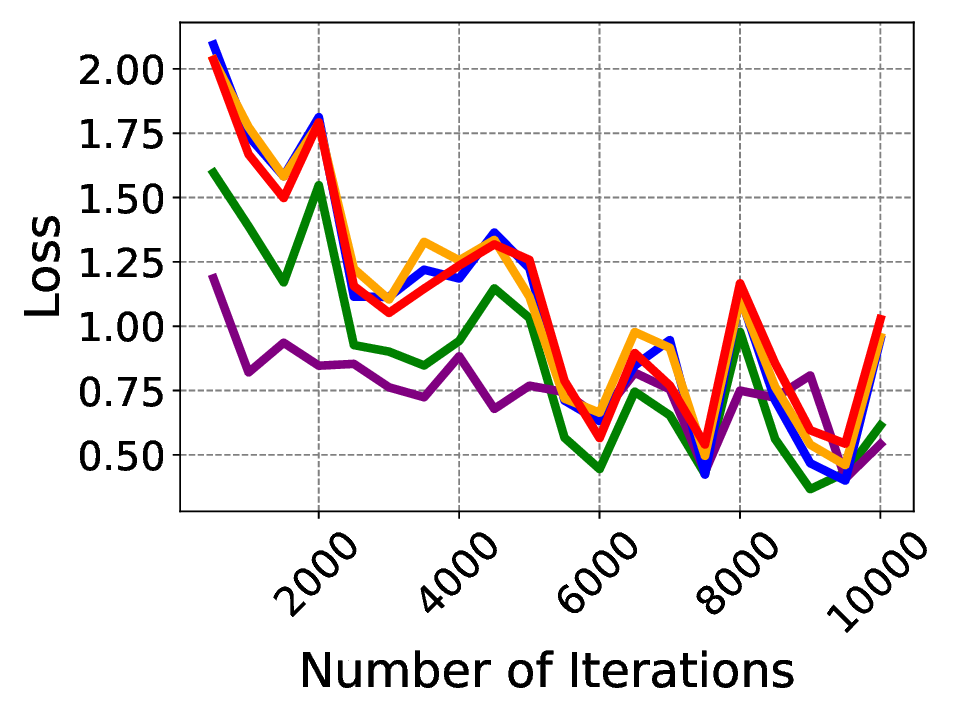}
    \end{subfigure}
    \begin{subfigure}
        \centering
        \includegraphics[width=.23\textwidth]{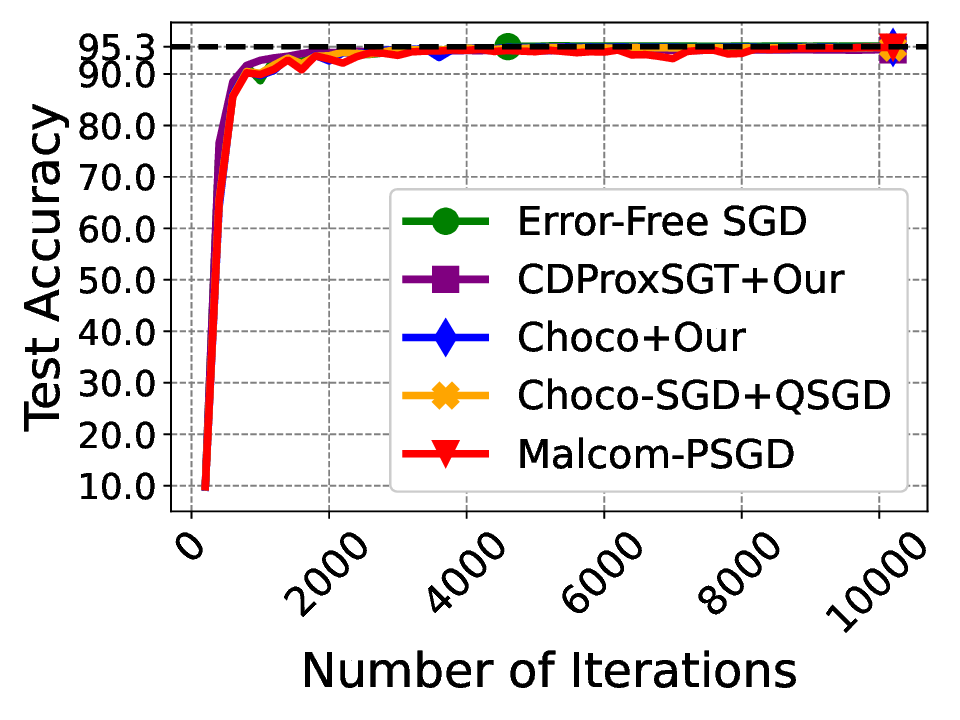}
    \end{subfigure}
    \begin{subfigure}
        \centering
        \includegraphics[width=.23\textwidth]{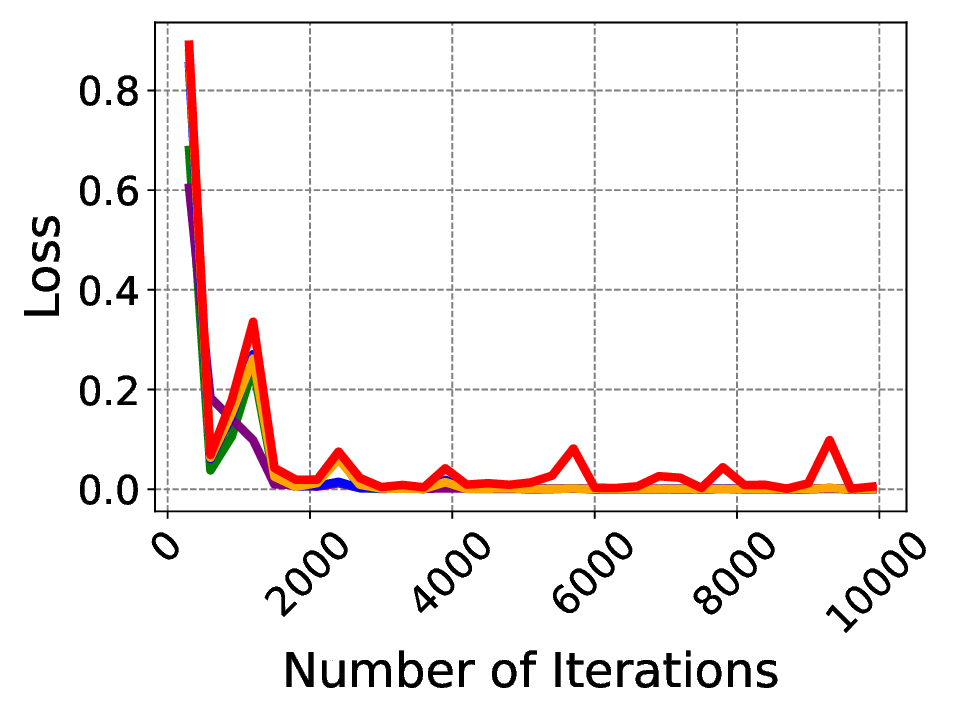}
    \end{subfigure}
    \caption{All plots are over the fully connected topology. The top row corresponds to the ResNet setup over the CIFAR10 dataset while the bottom corresponds to the 3-layer FC. NN over MNIST. The left column contains the accuracy plots while the right contains the loss plots. The horizontal black dashed line refers to the accuracy cut off and the diamond-colored points are where the bits were sampled from. For the CIFAR10 setup, there is an accuracy cut-off of 70.6, while the MNIST setup has a cut-off at 95.3.}
    \label{fig:fully_con}

\end{figure}
\textbf{The effects of Thresholding.} Figure \ref{fig:table} indicates that increasing the level of thersholding decreases the number of bits. While this matches our analysis in Section \ref{sec:bit_analysis}, the thresholding parameter is similar to other hyper parameters in that it must be tuned. Excessive thresholding reduces accuracy while too little does not provide stability under quantization nor a significant bit reduction. However, once an optimal threshold is found accuracy can be maintained while improving bit efficiency. This is evident from the figures and tables since when \textsc{Choco-SGD} was giving our compression method, \ourAlg used less bits for equivalent accuracy.      

\section{Conclusion}
\label{sec:conclusion}

We introduced the \ourAlg algorithm for decentralized learning with finite-sum, smooth, and non-convex loss functions. Our approach sparsifies local models by non-smooth $\ell_1$ regularization, rendering the implementation of the conventional SGD-based methods challenging. To address this challenge, we adopted the decentralized proximal SGD method to minimize the regularized loss, where the residuals of local SGD and proximal updates are computed prior to aggregation. Furthermore, we employed dithering-based gradient quantization and vector source coding schemes to compress model communication and leverage the low entropy of the updates to reduce the communication cost. By characterizing data sub-sampling and compression errors as perturbations in the proximal operation, we quantified the impact of gradient compression on training performance and established the convergence rate of \ourAlg with constant stepsizes. Moreover, we analyzed the communication cost in terms of the asymptotic code rate for the proposed algorithm. Numerical results validate the theoretical findings and demonstrate the improvement of our method in both learning performance and communication efficiency.

        

\bibliographystyle{iclr2023_conference}
\bibliography{example_paper}


\appendix

\section{Appendix}
\subsection{The Quantization Scheme in (\ref{eq:q(x)}) Satisfies Lemma \ref{lem:quant}}
\label{proof:quant} Let $\zeta(p,L)=\frac{1}{L}\lfloor pL+u\rfloor$ where $u$ uniformly drawn from [0,1]. Furthermore,
in the following section all expectations are taken with respect to the random dither $u$, for a given $\bm p$.
We denote the effective quantization scheme combining (\ref{eq:q(x)}) and the de-normalization procedure at the receiver end by $\hat{Q}(\cdot)$. Specifically, for any input vector $\pv\in\mathbb{R}^d$, the $i$-th entry of the output is given by
 $  \hat{Q}(p_i)=(\max{(\vp)}-\min{(\vp)})Q(p_i)+\min{(\vp)}$.
Let  $\hat{p}_i=\dfrac{p_i-\min(\vp)}{ \max(\vp) -\min(\vp)}$, $r_\vp= \max(\vp) -\min(\vp) $, and $m= \min(\vp)$. 
We have
$$\hat{Q}(p_i)=\dfrac{1}{\tau}\left(\zeta\left(\hat{p}_i, L\right)r_{\vp}+m\right).$$

 From \citep[Lemma 3.1]{qsgd}, the operator $\zeta(\cdot)$ is unbiased, i.e., $\E[\zeta(\hat{p}_i)]=\hat{p}_i$, where the expectation is taken w.r.t. the uniform dithering process. Moreover, let $l$ be the quantization region of $x_i$ such that $\hat{p}_i\in[l/L,(l+1/L)]$. Let $\mathcal{P}_i$ denote the probability that $\hat{p}_i$ is mapped to $(l+1)/L$, where we have $\mathcal{P}_i=\hat{p}_iL-l$.
In other words, we have $\zeta(\hat{p}_i,L)$ equal to $(l+1)/L$ with probability $\mathcal{P}_i$ and equal to $(l+1)/L$ with probability $1-\mathcal{P}_i$.
Therefore, we have
 \begin{align*}
     \E[\zeta(\hat{p}_i,L)^2]
     & =\E[\zeta(\hat{p}_i,L)]^2+\Var(\zeta(\hat{p}_i,L))=\hat{p}^{2}_{i}+\dfrac{\mathcal{P}(1-\mathcal{P})}{L^2}\leq \hat{p}^{2}_{i}+\dfrac{1}{4L^2},
 \end{align*}
where $\Var(\cdot)$ denotes the variance of the input random variable, and the last inequality follows from $\mathcal{P}(1-\mathcal{P})\leq 1/4$.

 Note that
 \begin{align}
     \E[\|\hat{Q}(\vx)\|^2]&=\frac{1}{\tau^2}\E\left[\sum^d_{i=1}\left(r_{\vp}\zeta(\hat{p}_i,L)+m\right)^2\right]\nonumber\\
    &=\frac{1}{\tau^2}\left(\sum^d_{i=1}\E\left[r^{2}_{\vp}\zeta(\hat{p}_i,L)^2\right]+\sum^d_{i=1}\E\left[2r_{\vp}m\zeta(\hat{p}_i,L)\right]+\sum^d_{i=1}m^2\right) \nonumber\\
     &\leq \frac{1}{\tau^2}\left(r^{2}_{\vp}\sum^d_{i=1}\left(\hat{p}^{2}_{i}+\dfrac{1}{4L^2}\right) +2r_{\vp}m\sum^d_{i=1}\hat{p}_i+dm^2 \right).\label{qx_pre_f_sub}
 \end{align}
For $\hat{p}$, we have
\begin{align}
    \sum^d_{i=1}\hat{p}_i&=\dfrac{1}{r_{\vp}}(-dm+\sum^d_{i=1}p_i), \label{f_sum}\\
    \sum^d_{i=1}\hat{p}^{2}_{i}
    &=\dfrac{1}{r^{2}_{\vp}}\sum^d_{i=1}p_i^2 -\dfrac{2m}{r^{2}_{\vp}}\sum^d_{i=1}p_i+\dfrac{dm^2}{r^{2}_{\vp}} .\label{f_sqr_sum}
\end{align}
Substituting (\ref{f_sum}) and (\ref{f_sqr_sum}) into \eqref{qx_pre_f_sub} we have:
\begin{align*}
    \E[\|\hat{Q}(\vp)\|^2]
    &=\frac{1}{\tau^2}\left(\|\vp\|_2^2+r^{2}_{\vp}\dfrac{d}{4L^2}\right)\overset{(a)}{\leq} \frac{1}{\tau^2}\left(1+\dfrac{d}{L^2}\right)\|\vp\|_2^2\overset{(a)}{=} \frac{\norm{\pv}_2^2}{\tau},
\end{align*}
where $(a)$ follows from $r^{2}_{\vp}\leq4\|\vp\|_\infty^2\leq4\|\vp\|_2^2$ and $(b)$ follows from the definition of $\tau$.
Finally, we have
 \begin{align*}
     \E\left[\|\hat{Q}(\vp)-\vp\|^2\right]
     &=\E\left[\sum^d_{i=1}\hat{Q}(p_i)^2\right]-2\E\left[\sum^d_{i=1}\hat{Q}(p_i)p_i\right]+\E\left[\sum^d_{i=1}p_i^2\right]\\
     &\leq \dfrac{1}{\tau}\|\vp\|^2-\dfrac{2}{\tau}\|\vp\|^2+\|\vp\|^2= \left(1-\dfrac{1}{\tau}\right)\|\vp\|^2,
 \end{align*}
 which completes the proof.
 \subsection{Matrix Representation of \ourAlg and Useful Lemmas}
To facilitate our proofs, we recast \ourAlg into an equivalent matrix form. Specifically, we define $\Yv^{(t)}=[\yv_1^{(t)},\cdots,\yv_n^{(t)}]$, $\mZ^{(t)}=[\vz_1^{(t)},\cdots,\vz_n^{(t)}]$, $\mQ^{(t)}=[\vq_1^{(t)},\cdots,\vq_n^{(t)}]$, ${\bf \xi}^{(t)}=[\xi_1^{(t)},\cdots,\xi_n^{(t)}]^T$ ,
and $\nabla F(\Xv^{(t)},\xi^{(t)})=[\nabla F_i(\xv_1^{(t)},\xi_1^{(t)}),\cdots,\nabla F_n(\xv_n^{(t)},\xi_n^{(t)})]$. Together with the matrices $\mX^{(t)}$ and $\overline \mX^{(t)}$ defined in Section \ref{sec:analysis}, we have the matrix form of Alg. \ref{alg:quant}, as shown in Alg. \ref{alg:quantmatrix}.

We also provide two useful results for the proof.
\begin{lemma}For any ${\bf A},{\bf B}$ with the same dimension, we have, for any $\alpha>0$:
		\begin{align}
			&\lVert{{\bf A}+{\bf B}}\rVert_F^2\leq (1+\alpha)\lVert{{\bf A}}\rVert_F^2+(1+\alpha^{-1})\lVert{{\bf B}}\rVert_F^2,\label{eq:a_plus_b}\\
    & \|\mA\mB\|_F \leq\|\mA\|_F\|\mB\|_2. \label{eq:a_times_b}
\end{align}
  \end{lemma}

\begin{lemma}
   For the mixing matrix $\mW$ satisfying Assumption \ref{as:mixMatrix}, we have, for any $k>0$,
\begin{align}
    & \left\|\mW^k-\frac{1}{n} {\bf 1}{\bf 1}^{\top}\right\|_2 \leq |\lambda_2|^k=(1-\delta)^k, \label{eq:w_minus_avg} \\
    & \left\|\Iv-\Wv\right\|_2^2=\lambda_n^2 \label{eq:i_minus_W},\\
    &\norm{\mI-\gamma\paren{\mW-\mI}}_2=1+\gamma\paren{1-\max_i\left\{\lambda_i\paren{\mW}\right\}},\label{eq:spect_other}
\end{align}
where we define $\delta:=1-|\lambda_2|$.
\end{lemma}
\begin{proof}
    See \citep[Lemma 16]{choco-convex}.
\end{proof}

\begin{algorithm}[!t]
  \caption{The matrix form of \ourAlg.
    \label{alg:quantmatrix}}
  \begin{algorithmic}[1]
    \State{\bfseries Initialize:\nonumber} $\mX^{(0)}, \mY^{(-1)} = 0$.
    \For{$t\in [0,\hdots, T-1]$} \Comment{All nodes $i$ do in parallel}
    \State{$\mX^{(t+1/2)}=\mX^{(t)}-\eta_t\nabla \bF\left(\mX^{(t)},{\bf \xi}^{(t)}\right) $}
    \State{$\mZ^{(t)}=\mathcal{S}_{\eta_t\mu}\left(\mX^{(t+1/2)}\right)$}\Comment{Proximal optimization by soft-thresholding; see \eqref{eq:soft_thres}.}
    \State{$\bm Q^{(t)}=Q(\mZ^{(t)}-\mY^{(t-1)}$)} \Comment{Residual quantization; see \eqref{eq:q(x)}.}
    \For{$j\in\mathcal{N}_i$}
    \State{\textbf{Encode and send} $\bm Q^{(t)}$} \Comment{Residual communication with source encoding from Alg. \ref{alg:encode}}.
    \State{\textbf{Receive and decode} $\bm Q^{(t)}$} \Comment{Receiver decoding.}
    \State{$\mY^{(t)}=\bm Q^{(t)}+\mY^{(t-1)}$}
    \EndFor
    \State{$\mX^{(t+1)}=\mZ^{(t)}+\gamma_t\mY^{(t)}\mW$}\Comment{Consensus aggregation; cf. (\ref{eq:conse}).}
    \EndFor
  \end{algorithmic}
\end{algorithm}

\subsection{Further Discussions on the Encoding Algorithm}
\label{app:encode}

\textbf{Preliminaries on Source Encoding.}
Elias coding \cite{elias1975universal} is a universal encoding scheme for unknown distributions that are assumed to generally have small integer values. Elias omega coding recursively encodes a value in binary: the string that encodes a value has the binary encoding of the value appended to itself, whose length becomes the subsequent value to be encoded in binary.

Golomb coding \cite{golomb1966run} is an encoding scheme parameterized by an integer $M$ that divides the value to be encoded. The quotient and remainder are separately encoded in unary and binary, respectively. 

\textbf{Algorithm Implementation}. The following algorithm describes the encoding scheme proposed in \cite{woldemariam2023low}. In this scheme, the type vector is assumed to be in descending order, while implementation allows for an unordered type vector. 
In Algorithm \ref{alg:encode}, the notation $[x]$ denotes the nearest integer of $x$. To decode, the first $L$ strings encoded with Elias omega coding are decoded with Elias decoding to retrieve the type vector. The run-lengths are then decoded with Golomb decoding and used to iteratively reconstruct the support vectors. With the positional information encoded through the support vectors, the decoder can fill in values of $\chi_{\ell}$ for all $\ell$.

\begin{algorithm}[!t]
    \caption{The source encoding scheme.}
    \label{alg:encode}
    \begin{algorithmic}[1]
\Function{Encode}{$\vq$, $L$}
\State Compute $\bm t(\vq)$ and encode it with Elias omega encoding.
\State Initialize $\mathcal{I} = \{ 1, \dots, N\}$.
\For{$\ell = 1, \dots, L-1$}
    \State $\mathcal{I}^\prime = \emptyset$, $\mathcal{R}_{\ell} = \emptyset$.
    \State Compute $M_{\ell} = [(\ln 2) (N-\sum_{m\leq \ell} t_m({\vq}))/t_\ell({\vq})]$.
    \State $r = 0$.
    \For{$j \in \mathcal{I}$}
    \If{$\delta(q_j - \ell) = 0$}
    \State $\mathcal{R}_{\ell} = \mathcal{R}_{\ell} \cup \{ r \}$.
    \If{$|\mathcal{R}_{\ell}| < t_{\ell}({\vq})$}
        \State Encode $r$ by Golomb coding with parameter $M_{\ell}$.
    \EndIf
    \State $r = 0$.
    \State $\mathcal{I}^\prime = \mathcal{I}^\prime \cup \{n\}$.
    \Else
    \State $r = r + 1$.
    \EndIf
    \State $\mathcal{I} = \mathcal{I}$ $ \backslash$ $\mathcal{I}^\prime$.
    \EndFor
\EndFor 
    \EndFunction
    \end{algorithmic}
\end{algorithm}

\subsubsection{Proof of Theorem \ref{th:bits}}
\label{proof:bits_num}
\begin{proof}
    We omit the training iteration index $t$ unless otherwise specified. Consider a scalar input $p$ to the proposed quantizer $Q(\cdot)$ and denote $q=Q(p)$ as the quantization output. Recall that the quantizer is given by

\begin{align}
    Q(p)=\frac{1}{\tau} \frac{\max(\vp)-\min(\vp)}{L}\left\lfloor L\left(\frac{p}{\max(\vp)-\min(\vp)}-\frac{\min(\vp)} {\max(\vp)-\min(\vp)} \right)\right\rfloor+\min(\vp),
\end{align}
where $L$ is the number of levels, and $\tau=1+\frac{d}{L^2}$. 
Furthermore let $r_{\vp} =\max({\vp^{(t)}})-\min({\vp^{(t)}})$
We assume for simplicity that $L$ is an odd number. As $d\to\infty$,
\begin{align}
    Q(p)&=\frac{1}{\tau} \frac{r_{\vp}}{L}\left\lfloor\frac{L p}{r_{\vp}}+0.5L\right\rfloor+\min(\pv)\nonumber\\
    &=\frac{1}{\tau} \frac{r_{\vp}}{L}\left\lfloor\frac{L p}{r_{\vp}}+0.5\right\rfloor+\frac{(L-1)(r_{\vp})}{2L \tau}+\min(\pv)\nonumber\\
    &=\frac{1}{\tau} \frac{r_{\vp}}{L}\text{round}\left(\frac{L p}{r_{\vp}}\right)+\frac{(L-1)(r_{\vp})}{2L \tau}+\min(\pv)
\end{align}
This quantizer depends on the two random values
$\overline{p}=\max{\bm p}$ and $\underline{p}=\min{\bm p}$.
First, note that an affine mapping of the output does not change its entropy, so we can study a quantizer that does not center the distribution, since its entropy would be the same as that of $Q(p)$ equals to the entropy of the following quantizer:
\begin{align}
     \tilde Q(p)=\frac{r_{\vp}}{L}\text{round}\left(\frac{L p}{r_{\vp}}\right).
\end{align}
We note that, for the Laplace distribution, we have the following concentration inequality 
 $$\Pr(\max_{1\leq j\leq d}|p_j|>\vartheta \rho)\leq de^{-\vartheta},\forall \vartheta>0.$$
   Setting $\epsilon=de^{-\vartheta}$, we have
\begin{align}\label{eq29}
    \Pr\left(\max_{1\leq j\leq d}|p_j|>(\ln d-\ln \epsilon )\rho_t\right)\leq\epsilon.
\end{align}
Since the input vector has i.i.d. entries and the Laplace distribution is symmetric, we have $r_{\pv}=2\max_{1\leq j\leq d}|q_j|$. Combining this result with \eqref{eq29}, with probability at least $1-\epsilon$ for a small $\epsilon<1$, we have
\begin{align}\label{eq29a}
   \frac{r_{\pv}}{\rho^{(t)}}< 2\ln(d/\epsilon).
\end{align}
 We therefore choose to study a simplififed scenario where the quantizer is adaptive with respect to $\rho^{(t)}$, as opposed to the outcomes of the random vector $\bm p^{(t)}$ and where we choose a time variant but non-random range $r^{(t)} \propto \rho^{(t)}2\ln(d/\epsilon)$ that shrinks as the expected values of the residual entries decrease.

Assume that $p$ follows a zero-mean Laplace distribution with parameter $\rho$. $f(p)=\frac{1}{2\rho}e^{-|p|/\rho}$ and $F(p)=1-e^{-x/\rho}/2$ if $p\geq 0$.
Recall that our bound on the number of bits $d\cdot R(L)$ is given by:
\begin{align}
    d\left(H(q) + 2.914(1 - f_0) + f_0\log_2f_0 + \sum_{\ell=1}^{L-1}f_{\ell}\log_2(1 - \sum_{m=0}^{\ell -1}f_m)\right).
\end{align}
where $f_\ell$ is the shuffled PMFs of the quantized values in the descending order. We shall compute them one by one as follows.
First, omitting the dependence on $t$, the value of $f_0$ is given by
\begin{align}
    f_0&=\Pr(q=0)=\Pr(-\frac{r}{2L}\leq p\leq \frac{r}{2L})\nonumber\\
    &=2(\Pr(p\leq \frac{r}{2L})-0.5)\nonumber\\
    &=1-e^{-\Delta},
\end{align}
where $\Delta=\frac{r}{2L\rho}$.
Furthermore, since the Laplace distribution is symmetric around zero, we have
\begin{align}
    H(q) + f_0\log_2f_0&=-\sum_{\ell=1}^{L-1} f_\ell \log f_\ell\nonumber=-2\sum_{\ell=1}^{(L-1)/2} \tilde f_\ell \log \tilde f_\ell,
\end{align}
where 
\begin{align}
    \tilde f_\ell&=\Pr\left(\frac{(2\ell-1)(r_{\vp})}{2L}\leq q\leq \frac{(2\ell+1)(r_{\vp})}{2L}\right)\nonumber\\
    &=\Pr\left((2\ell-1)\Delta\rho\leq q\leq (2\ell+1)\Delta\rho\right)\nonumber\\
    &=\frac{1}{2}\left(e^{-(2\ell-1)\Delta}-e^{-(2\ell+1)\Delta}\right)\nonumber\\
    &=\sinh(\Delta)e^{-2\ell \Delta}.
\end{align}
Therefore, 
\begin{align}
    &H(q) + f_0\log_2f_0=-2\sum_{\ell=1}^{(L-1)/2} \tilde f_\ell \log \tilde f_\ell\nonumber\\
    &=2\sinh(\Delta)\sum_\ell -e^{-2 \Delta \ell}\log (\sinh(\Delta)e^{-2\ell \Delta})\nonumber\\
    &=2\sinh(\Delta)\sum_\ell -e^{-2 \Delta \ell}\log (e^{-2\ell \Delta})-2\sinh(\Delta)\log (\sinh(\Delta)\sum_\ell (e^{-2 \Delta})^ \ell\nonumber\\
\end{align}
On the other hand,
\begin{align}
    &\sum_{\ell=1}^{L-1}f_{\ell}\log_2(1 - \sum_{m=0}^{\ell -1}f_m)=\sum_{\ell=1}^{L-1}f_{\ell}\log_2(f_\ell+\cdots+f_{L-1})\nonumber\\
    &\leq 2\sum_{\ell=1}^{(L-1)/2} \tilde f_\ell \log(2\sum_{m=\ell}^{(L-1)/2}\tilde f_m)\nonumber\\
    &=2\sum_{\ell} \tilde f_\ell +2\sum_{\ell}\tilde f_\ell\log(\sum_{m=\ell}^{(L-1)/2}\tilde f_m)\nonumber\\
    &=2\sinh(\Delta)\sum_{\ell}e^{-2\ell \Delta}+2\sinh(\Delta)\log \sinh(\Delta)\sum_{\ell}e^{-2\ell \Delta}+2\sinh(\Delta)\sum_{\ell}e^{-2\ell \Delta}\log(\frac{1-e^{-2 \Delta(L/2-1/2-\ell+1)}}{1-e^{-2 \Delta}}e^{-2\Delta\ell})\nonumber\\
    &= 2\sinh(\Delta)(1+\log \sinh(\Delta)-\log (1-e^{-2\Delta}))\sum_{\ell}e^{-2\ell \Delta}+2\sinh(\Delta)\sum_{\ell}e^{-2\ell \Delta}\log(e^{-2\Delta\ell}-e^{-\Delta(L+1)})\nonumber\\
    &\leq 2\sinh(\Delta)(1+\log \sinh(\Delta)-\log (1-e^{-2\Delta}))\sum_{\ell}e^{-2\ell \Delta}+2\sinh(\Delta)\sum_{\ell}e^{-2\ell \Delta}\log(e^{-2\Delta\ell})
\end{align}

Summing all the terms,
\begin{align}
   & R(L) \leq 2.914e^{-\Delta}+2\sinh(\Delta)(1-\log (1-e^{-2\Delta}))\sum_{\ell}e^{-2\ell \Delta}\nonumber\\
    &\leq 2.914e^{-\Delta}+2\sinh(\Delta)\frac{1-\log (1-e^{-2\Delta})}{e^{2\Delta}-1}
    \end{align}
    which proves the theorem.

\end{proof}

\subsection{Proof of Theorem \ref{th:converge}}\label{app:proof_converge}
Before proceeding with the proof we introduce some preliminearies for Proximal-SGD. 
We assume $\mathcal{F}(\vx)=\frac{1}{n}\sum_{i=1}^nF_i(\vx)+h(\vx)$, where $h(\cdot):\R^d\rightarrow \R^d$ is a convex function and where $\norm{h(\vx)}^2\leq B^2$. While in the case of \ourAlg we have defined the proximal operator as  
\begin{align}
    =\argmin_{\vu}\left\{\mu\eta_t\norm{\vu}_1+{\dfrac{1}{2}\norm{\vu-\vz_i^{(t+1)}}^2}\right\},
\end{align}
the more generic proximal operator is defined as
\begin{align}
\arg \min _{\vu \in \R^d}\left\{\langle \vg, \vu\rangle+\frac{1}{\eta} V(\vu, \vx)+h(\vu)\right\},
\end{align} where
\begin{align}  
&V(\vx, \vv)=\theta(\vx)-[\theta(\vz)+\langle\nabla \theta(\vz), \vx-\vz\rangle]\nonumber. 
\end{align}
Here $\theta:\R^d\rightarrow \R$ is a distance generating function satisfying the following:
\begin{align}
\langle \vx-\vv, \nabla \theta(\vx)-\nabla \theta(\vv)\rangle \geq \alpha\|\vx-\vv\|^2, \quad \forall \vx, \vv \in \R^d 
\end{align}, where $\alpha$ is the modulus of strong convexity for $\theta$.
As is standard for Proximal SGD proofs we define the projected gradient with respect to the true function($\nabla F_i(\vx_i)$) as:
\begin{align}
    \vg_i(\vx_i)=\frac{1}{\eta}\left(\vx_i-\vz_i\right).
\end{align} Additionally we define the stochastic projected gradient with respect to $\nabla f_i(\vx_i,\bm{\xi}_i)$ 
\begin{align}
\label{eq:g_tilde}
    \tilde{\vg}_i(\vx_i)=\frac{1}{\eta}\left(\vx_i-\vz_i\right).
\end{align}

Using these above definitions we introduce the following Lemma.
 Before proceeding we introduce the following Lemma \citep{ghadimi2016mini}[Lemma 1]
 \begin{lemma}
     \label{lem:grad_inner_prod_bound}
     If $\vz=\operatorname{prox}_{\eta,h}(\vx)$ and $\tilde{\vg}$ in \eqref{eq:g_tilde}, then for $\forall \vx_i\mX$, $\eta>0$, and $\nabla \bm{F}\in\R^{d}$
     \begin{align}
         \left\langle \nabla\bm{F}(\xv),\vg\left(\vx\right)\right\rangle\geq \alpha\norm{\vg(\vx)}^2+\frac{1}{\eta}\left(h(\vz)-h(\vx)\right)
     \end{align}
 \end{lemma}
In order to prove Theorem \ref{th:converge} we need the following consensus lemma.

\begin{lemma}
    \label{th:consensus} Suppose Assumptions \ref{as:mixMatrix}-\ref{as:nice_f} hold. If 
\begin{align}\label{eq13}
\vx_i^{(0)}={\bf 0},\forall i,
\end{align}
where $\tau$ is defined in Lemma \ref{lem:quant}. 
Then, for $\forall t>0$
Lemma \ref{lem:quant} holds:
 	\begin{align}\label{eq_consensus2}
	  &\sum_{i=1}^n  \E\|\vx_i^{(t)}-\overline\vx^{(t)}\|^2\leq \frac{24}{\omega^2}(2G^2+2\sigma^2+\mu^2d)n\eta^2,
	\end{align}
 where $\omega=\frac{(1-|\lambda_2|)^2}{82\tau}$.
\end{lemma}
For a proof of Lemma \ref{th:consensus} see Appendix \ref{proof:lemma_con}. Finally, with the above Lemmas we can prove a more generic theorem that implies Theorem \ref{th:converge}. 
\begin{theorem}
    \label{th:int_converge}
    Suppose  $\mathcal{F}(\vx)=\frac{1}{n}\sum_{i=1}^nF_i(\vx)+h(\vx)$, where $F_i$ is as defined in \eqref{eq:obj_func} where $h(\cdot):\R^d\rightarrow \R$ is a convex function and where $\norm{h(\vx)}^2\leq B^2$. If Assumptions \ref{as:mixMatrix}-\ref{as:F_unBiased} are met, then the projected gradient diminishes by:
    \begin{align}
{T+1}\sum_{t=0}^{T}\norm{\vg\left(\bar{\vx}^{(t)}\right)}^2&\leq 2\sqrt{\left(\mathcal{F}\left(\bar{\vx}^{(0)}\right)-\mathcal{F}^*\right)\dfrac{K\left(\sigma^2+4B^2+G\sqrt{2G^2+2\sigma^2+B^2}\right)}{(T+1)n\omega\alpha^2}}\nonumber\\
    &~~~~+7\left(\dfrac{K\sqrt{2G^2+2\sigma^2+B^2}\left(\mathcal{F}\left(\bar{\vx}^{(0)}\right)-\mathcal{F}^*\right)}{\alpha(T+1)\omega}\right)^{\frac{2}{3}}+\dfrac{16K\left(\mathcal{F}\left(\bar{\vx}^{(0)}\right)-\mathcal{F}^*\right)}{T+1}
    \end{align}
\end{theorem} 
\begin{proof}
    Utilizing Lipschitz smoothness and \eqref{eq:g_tilde} we have:

\begin{align}
 \E_{t+1}\bm{F}\left(\bar{\vx}^{(t+1)}\right) &\leq \bm{F}\left(\bar{\vx}^{(t)}\right)+\E_{t+1}\left\langle\nabla \bm{F}\left(\bar{\vx}^t\right), \bar{\vx}^{(t+1)}-\bar{\vx}^{(t)}\right\rangle+\E_{t+1}\frac{K}{2}\left\|\bar{\vx}^{(t+1)}-\bar{\vx}^{(t))}\right\|^2 \nonumber\\
& \overset{\eqref{eq:g_tilde}}{=}\bm{F}\left(\bar{\vx}^{(t)}\right)-\eta\E_{t+1}\left\langle\nabla \bm{F}\left(\bar{\vx}^{(t)}\right),\frac{1}{n}\sum_{i=1}^n\tilde{\vg}_i\left(\vx^{(t)}_i\right)\right\rangle+\frac{\eta^2K}{2}\E_{t+1}\left\| \frac{1}{n}\sum_{i=1}^n\tilde{\vg}_i\left(\vx^{(t)}_i\right)\right\|^2 \nonumber\\
&=\bm{F}\left(\bar{\vx}^{(t)}\right)-\eta\E_{t+1}\left\langle\nabla\bm{F}\left(\bar{\vx}^{(t)}\right),\tilde{\vg}\left(\bar{\vx}^t\right)\right\rangle +\frac{\eta^2K}{2}\E_{t+1}\norm{\sum_{i=1}^n\tilde{\vg}_i\left(\vx^{(t)}_i\right)}^2\nonumber\\
&~~~~~~~+\eta\E_{t+1}\left\langle \nabla\bm{F}\left(\bar{\vx}^{(t)}\right),\tilde{\vg}\left(\bar{\vx}^{(t)}\right)-\frac{1}{n}\sum_{i=1}^n\tilde{\vg}_i\left(\vx^{(t)}_i\right)\right\rangle\nonumber\\
&\overset{\text{Lemma } \ref{lem:grad_inner_prod_bound}}{\leq}\bm{F}\left(\bar{\vx}^{(t)}\right)-\eta\E_{t+1}\left(\alpha\norm{\tilde{\vg}\left(\bar{\vx}^{(t)}\right)}^2+\frac{1}{\eta}\left(h(\bar{\vz}^{(t+1)})-h(\bar{\vx}^{(t)})\right)\right)+\frac{\eta^2K}{2}\E_{t+1}\norm{\sum_{i=1}^n\tilde{\vg}_i\left(\vx^{(t)}_i\right)}^2\nonumber\\
&~~~~~~~~~~+\eta\E_{t+1}\norm{\nabla\bm{F}\left(\bar{\vx}^{(t)}\right)}\norm{\tilde{\vg}\left(\bar{\vx}^{(t)}\right)-\frac{1}{n}\sum_{i=1}^n\tilde{\vg}_i\left(\vx^{(t)}_i\right)}\nonumber \\
\E_{t+1}\mathcal{F}\left(\bar{\vx}^{(t+1)}\right)&=\mathcal{F}\left(\bar{\vx}^{(t)}\right)-\eta\alpha\E_{t+1}\norm{\tilde{\vg}\left(\bar{\vx}^{(t)}\right)}^2+\frac{\eta^2K}{2}\E_{t+1}\norm{\sum_{i=1}^n\tilde{\vg}_i\left(\vx^{(t)}_i\right)}^2\nonumber\\
&~~~~~~~~~+\E_{t+1}\eta\norm{\nabla\bm{F}\left(\bar{\vx}^{(t)}\right)}\norm{\tilde{\vg}\left(\bar{\vx}^{(t)}\right)-\frac{1}{n}\sum_{i=1}^n\tilde{\vg}_i\left(\vx^{(t)}_i\right)}\nonumber
\end{align}
Note that \begin{align}
\label{eq:Ksmooth_proj}
    \norm{\vg(\vx)-\vg(\vy)}\leq\frac{1}{\alpha}\norm{\nabla \bm{F}(\vx)-\nabla \bm{F}(\vy)}\leq \frac{K}{\alpha}\norm{\vx-\vy}
\end{align}
Using \eqref{eq:Ksmooth_proj}, the $K$ smoothness of $\bm{F}$, and the definition of variance we have:
\begin{align}
\E_{t+1}\mathcal{F}\left(\bar{\vx}^{(t+1)}\right)&\leq\mathcal{F}\left(\bar{\vx}^{(t)}\right)-\eta\alpha\norm{\tilde{\vg}\left(\bar{\vx}^{(t)}\right)}^2+\frac{\eta GK}{\alpha n}\sum^{n}_{i=1}\norm{\bar{\vx}^{(t)}-\vx_i^{(t)}}\nonumber\\
&~~~~~~~~
+\frac{\eta^2K}{2}\left(\E_{t+1}\left[\norm{\frac{1}{n}\sum_{i}^n\left(\tilde{\vg}_i\left(\vx^{(t)}_i\right)-\vg_i\left(\vx^{(t)}_i\right)\right)}\right] +\norm{\frac{1}{n}\sum_{i=1}^n\vg_i\left(\vx_i^{(t)}\right)- \vg\left(\bar{\vx}^{(t)}\right)+\vg\left(\bar{\vx}^{(t)}\right)}^2_2\right)\nonumber
\end{align}
By Assumption \ref{as:nice_f} and since $\norm{h(\vx)}^2\leq B^2$ we know that $\Var\left[\tilde{\vg}_i\left(\vx_i\right)\right]\leq \sigma^2+4B^2$ we have:
\begin{align}
\E_{t+1}\mathcal{F}\left(\bar{\vx}^{(t+1)}\right)&\leq\mathcal{F}\left(\bar{\vx}^{(t)}\right)-\eta\alpha\norm{\tilde{\vg}\left(\bar{\vx}^{(t)}\right)}^2+\frac{\eta GK}{\alpha n}\sum^{n}_{i=1}\norm{\bar{\vx}^{(t)}-\vx_i^{(t)}}\nonumber\\
&~~~~~~~~ + \frac{\eta^2K}{2}\left[\frac{\sigma^2+4B}{n}+\frac{2}{n}\sum_{i=1}^n\norm{\vg\left(\bar{\vx}^{(t)}\right)-\vg_i\left(\vx^{(t)}_i\right)}^2+2\norm{\vg\left(\bar{\vx}^{(t)}\right)}^2\right]\nonumber
\end{align}
By \eqref{eq:Ksmooth_proj} and Lemma \ref{th:consensus}
\begin{align}
\E_{t+1}\mathcal{F}\left(\bar{\vx}^{(t+1)}\right)&\leq\mathcal{F}\left(\bar{\vx}^{(t)}\right)-\eta\alpha\left(1-K\eta\right)\norm{\tilde{\vg}\left(\bar{\vx}^{(t)}\right)}^2+\frac{\eta^2 K(\sigma^2+4B)}{n} \nonumber \\
&~~~~~~~~+\dfrac{\eta^2\sqrt{24}GK}{\alpha\omega\sqrt{n}}\sqrt{2G^2+2\sigma^2+B^2}+\frac{\eta^4K^324\left(2G^2+2\sigma^2+B^2\right)}{\omega^2}\nonumber
\end{align}
Finally, utilizing that $\eta\leq \frac{1}{4K}$ we have
\begin{align}
\E_{t+1}\mathcal{F}\left(\bar{\vx}^{(t+1)}\right)&\leq\mathcal{F}\left(\bar{\vx}^{(t)}\right)-\frac{\eta}{2\alpha}\norm{\tilde{\vg}\left(\bar{\vx}^{(t)}\right)}^2+\dfrac{\eta^2K\left(\sigma^2+4B^2+G\sqrt{2G^2+2\sigma^2+B^2}\right)}{n\omega\alpha}+\frac{\eta^3K^26\left(2G^2+2\sigma^2+B^2\right)}{\omega^2}\nonumber
\end{align}
$\implies$
\begin{align}
\dfrac{1}{T+1}\sum_{t=0}^{T}\norm{\vg\left(\bar{\vx}^{(t)}\right)}^2&\leq\dfrac{2\left(\mathcal{F}\left(\bar{\vx}^{(0)}\right)-\mathcal{F}^*\right)}{\alpha\eta(T+1)}+\dfrac{\eta2K\left(\sigma^2+4B^2+G\sqrt{2G^2+2\sigma^2+B^2}\right)}{n\omega\alpha}+\frac{\eta^2K^212\left(2G^2+2\sigma^2+B^2\right)}{\omega^2}
\end{align}

Finally, by applying \citep{choco}[Lemma A.4] we have
\begin{align}
\dfrac{1}{T+1}\sum_{t=0}^{T}\norm{\vg\left(\bar{\vx}^{(t)}\right)}^2&\leq 2\sqrt{\left(\mathcal{F}\left(\bar{\vx}^{(0)}\right)-\mathcal{F}^*\right)\dfrac{K\left(\sigma^2+4B^2+G\sqrt{2G^2+2\sigma^2+B^2}\right)}{(T+1)n\omega\alpha^2}}\nonumber\\
    &~~~~+7\left(\dfrac{K\sqrt{2G^2+2\sigma^2+B^2}\left(\mathcal{F}\left(\bar{\vx}^{(0)}\right)-\mathcal{F}^*\right)}{\alpha(T+1)\omega}\right)^{\frac{2}{3}}+\dfrac{16K\left(\mathcal{F}\left(\bar{\vx}^{(0)}\right)-\mathcal{F}^*\right)}{T+1}
\end{align}
\end{proof}
Theorem \ref{th:int_converge} directly implies Theorem \ref{th:converge}. Additionally it gives us the following corollary.

\begin{corollary}
    Under Assumptions \ref{as:mixMatrix}-\ref{as:F_unBiased} and $\mathcal{F}$ as defined \eqref{eq:obj_func}, then \ourAlg converges with rate $\mathcal{O}\left(1\sqrt{nT\omega}+1/(T\omega)^{2/3}\right)$. 
\end{corollary}
\begin{proof}
    Apply Theorem \ref{th:converge} and note that $h=\mu\norm{\cdot}_1$ and that since we are using the Euclidean projection $\alpha=1$.
\end{proof}

\subsection{Lemma and Proposition Proofs}

\subsubsection{Proof of Lemma \ref{th:consensus}}
\label{proof:lemma_con}
\begin{proof}
    Define $\Phi^{(t+1)}\in\partial \left(\mu\lnorm\Xv^{(t+1)}\rnorm_{1,1}\right)$.
We observe the following property: 
{\observation{Every subgradient of the regularization function $\phiv\in \partial (\mu\norm{\xv}_{1})$ has a bounded norm as $\norm{\phiv}_2\leq \mu \sqrt{d}$. Under Assumption \ref{as:nice_f},  we have

 $\E[\norm{\nabla F(\Xv;{\bf \xi})+\Phiv}_F^2]\leq 2n(2G^2+2\sigma^2+\mu^2d)$ for any $\phiv_i\in \partial(\mu\norm{\xv_i}_1)$, $\Xv$, and $\bf \xi$.
}}

Define a sequence $\{r_t\}_t$ such that
\begin{align}\label{appa_eq03}
    r_t=\E\left[\lnorm\Xv^{(t)}-\overline \Xv^{(t)}\rnorm_F^2+\lnorm\Xv^{(t)}- \Yv^{(t)}\rnorm_F^2\right].
\end{align}
In order for us to determine the consensus rate of \ourAlg we need our aggregation scheme to satisfies the following(Assumption 3 \cite{choco}):
\begin{assumption}
\label{as:layp}
For an averaging scheme $h: \mathbb{R}^{d \times n} \times \mathbb{R}^{d \times n} \rightarrow \mathbb{R}^{d \times n} \times \mathbb{R}^{d \times n}$ let $\left(\mX^{+}, \mY^{+}\right):=$ $h(\mX, \mY)$ for $\mX, \mY \in \mathbb{R}^{d \times n}$. Assume that $h$ preserves the average of iterates:
$$
\mX^{+} \frac{11^{\top}}{n}=X \frac{11^{\top}}{n}, \quad \forall \mX, \mY \in \mathbb{R}^{d \times n}
$$
and that it converges with linear rate for a parameter $0<c \leq 1$
$$
\mathbb{E}_h \Psi\left(\mX^{+}, \mY^{+}\right) \leq(1-c) \Psi(\mX, \mY), \quad \forall \mX, \mY \in \mathbb{R}^{d \times n},
$$
and Laypunov function $\Psi(\mX, \mY):=\|\mX-\bar{\mX}\|_F^2+\|\mX-\mY\|_F^2$ with $\bar{\mX}:=\frac{1}{n} X 11^{\top}$, where $\mathbb{E}_h$ denotes the expectation over internal randomness of averaging scheme $h$.
\end{assumption}
Our aggregation scheme satisfies Assumption \ref{as:layp} with:
\begin{align}
\E_Q\norm{\mX^{(t+1)}-\overline{\mX}^{(t+1)}}_F^2+\norm{\mX^{(t+1)}+\mY^{(t+1)}}_F^2\leq (1-\omega)\left[\norm{\overline{\mZ}^{(t)}-\mZ^{(t)}}_F^2+\norm{\mY^{(t)}-\mZ^{(t)}}_F^2\right]
\end{align}
where $\omega=\frac{(1-|\lambda_2|)^2}{82\tau}$. See \citep[Proof of Theorem 2]{choco-convex} for the details and note that $\mX=\mZ$ and $\hat{\mX}=\mY$. 
Therefore, at iteration $t+1$, we have 
\begin{align*}
   r_{t+1} & \leq \E\left[(1-\omega)\left(\norm{\overline{\mZ}^{(t)}-\mZ^{(t)}}_F^2+\norm{\mY^{(t)}-\mZ^{(t)}}\right)\right]\\
   &=(1-\omega)\E\left[\left(\mX^{(t)}-\eta\left(\nabla\bm{F}\left(\mX^{(t)},\bm{\xi}^{(t)}\right)\right)\right)\left(\frac{\vone\vone^{\top}}{n}-\mI\right)\right]\nonumber\\
   &~~~~~+(1-\omega)\E\left[\norm{\mY^{(t)}-\left(\mX^{(t)}-\eta\left(\nabla\bm{F}\left(\mX^{(t)},\bm{\xi}^{(t)}\right)\right)\right)}\right]\\
   &=(1-\omega)\E\left[\operatorname{prox}_{\eta,g}\left(\mX^{(t)}-\eta\left(\nabla\bm{F}\left(\mX^{(t)},\bm{\xi}^{(t)}\right)\right)\right)\left(\frac{\vone\vone^{\top}}{n}-\mI\right)\right]\nonumber\\
   &~~~~~+(1-\omega)\E\left[\norm{\mY^{(t)}-\operatorname{prox}_{\eta,g}\left(\mX^{(t)}-\eta\left(\nabla\bm{F}\left(\mX^{(t)},\bm{\xi}^{(t)}\right)\right)\right)}\right]\\
   &\leq(1-\omega)(1+\alpha^{-1})\E\left[\norm{\overline{\mX}^{(t)}-\mX^{(t)}}_F^2+\norm{\mY^{(t)}-\mX^{(t)}}^2_F\right]\\
   &~~~~~+(1-\omega)(1+\alpha)\eta^2\E\left[\norm{\left(\nabla\bm{F}\left(\mX^{(t)},\bm{\xi}^{(t)}\right)+\bm{\Phi}^{(t)}\right)\left(\frac{\vone\vone^{\top}}{n}-\mI\right)}_F^2+\norm{\nabla\bm{F}\left(\mX^{(t)},\bm{\xi}^{(t)}\right)+\bm{\Phi}^{(t)}}_F^2\right]\\
   &\overset{\eqref{eq:a_times_b}}{\leq}(1-\omega)\left((1+\alpha^{-1})\E\left[\norm{\overline{\mX}^{(t)}-\mX^{(t)}}_F^2+\norm{\mY^{(t)}-\mX^{(t)}}^2_F\right]+\eta^2(1+\alpha)4n\left(2G^2+2\sigma^2+\mu^2d\right)\right)\\
   &\overset{\alpha=\frac{2}{\omega}}{\leq}\left(1-\frac{\omega}{2}\right)\E\left[\norm{\overline{\mX}^{(t)}-\mX^{(t)}}_F^2+\norm{\mY^{(t)}-\mX^{(t)}}^2_F\right]+\frac{2}{\omega}\eta^26n\left(2G^2+2\sigma^2+\mu^2d\right)
\end{align*}
Letting $A=6n\left(2G^2+2\sigma^2+\mu^2d\right)$ and noticing that we have the following recursion:
\begin{align}
     r_{t+1}\leq (1-\omega/2)r_t+\frac{2}{\omega}\eta_t^2A,\nonumber
\end{align}
Since $r_0\leq\eta^2\frac{4}{\omega^2}A$ and $\overline{\mX}^{(0)}=\mX^{(0)}=\mY^{(0)}=0$ we have that:
\begin{align}
    r_{t+1}\leq\left(1-\frac{\omega}{2}\right)r_t+\frac{2}{\omega}\eta^2A\leq\left(1-\frac{\omega}{2}\right)\frac{4}{\omega^2}\eta^2A+\frac{2}{\omega}\eta^2A=\frac{4}{\omega^2}\eta^2A
\end{align}
\end{proof}

\section{Experimental Results}
\label{app:add_exp}
Unless otherwise stated, for all of the algorithms we fixed the time-varying parameters and only considered their constant variants. This allows for a more direct comparison with the state of the art since both \cite{choco} and \cite{sparsesgt} utilize constant hyper-parameters. We perform hyperparameter tunning over the following sets: $\gamma=\{1,.8,.6\}$ and $\eta=\{.02,.05,.1,.15,,2\}$. We set the mini-batch size to 20 and 64 to be consistent with the prior art. Furthermore, we fix the precision level for each algorithm so that $L=8$. For \textsc{Choco-SGD} and \ourAlg we found $\gamma=1$,$\eta=.2$ For CDProx-SGT $\gamma=1$ and $\eta=.02$. For \ourAlg we used $\mu=7e-6$ while CDProxSGT used $\mu=7e-8$ We train for a total of 10,000 iterations and run the experiments multiple times taking the average. 

\subsection{Ring Topology}
For both experiments, we considered a ``ring-like" networking topology composed of 10 nodes. The network topology and the corresponding mixing matrix are shown in Figure \ref{fig:ringtop}.  
\subsubsection{CIFAR10}
Under this setup, each node was given an i.i.d. subset of the 50,000 training samples. Each node has the untrained ResNet18 model.

\subsubsection{MNIST}
Under this set up each node was given an i.i.d subset of the 60,000 training samples. The DNN consists of three fully connected layers with total $d=669,706$ parameters.
\subsection{Fully Connected Topology}
Here, we consider a fully-connected network topology, i.e., $\Wv={\bf 1}{\bf 1}^T/n$.

\subsubsection{CIFAR10}
Due to memory constraints, we only considered a network with 5 nodes instead of 10.
\subsubsection{MNIST}
This setup is identical to the ring topology except we changed the learning rate for CDProxSGT from $\eta=.02$ to $\eta=.2$. CDProxSGT would not converge with smaller learning rates.

\begin{figure}[!t]

	\begin{minipage}{0.49\linewidth}
		\centering
		\includegraphics[width=3.2 in]{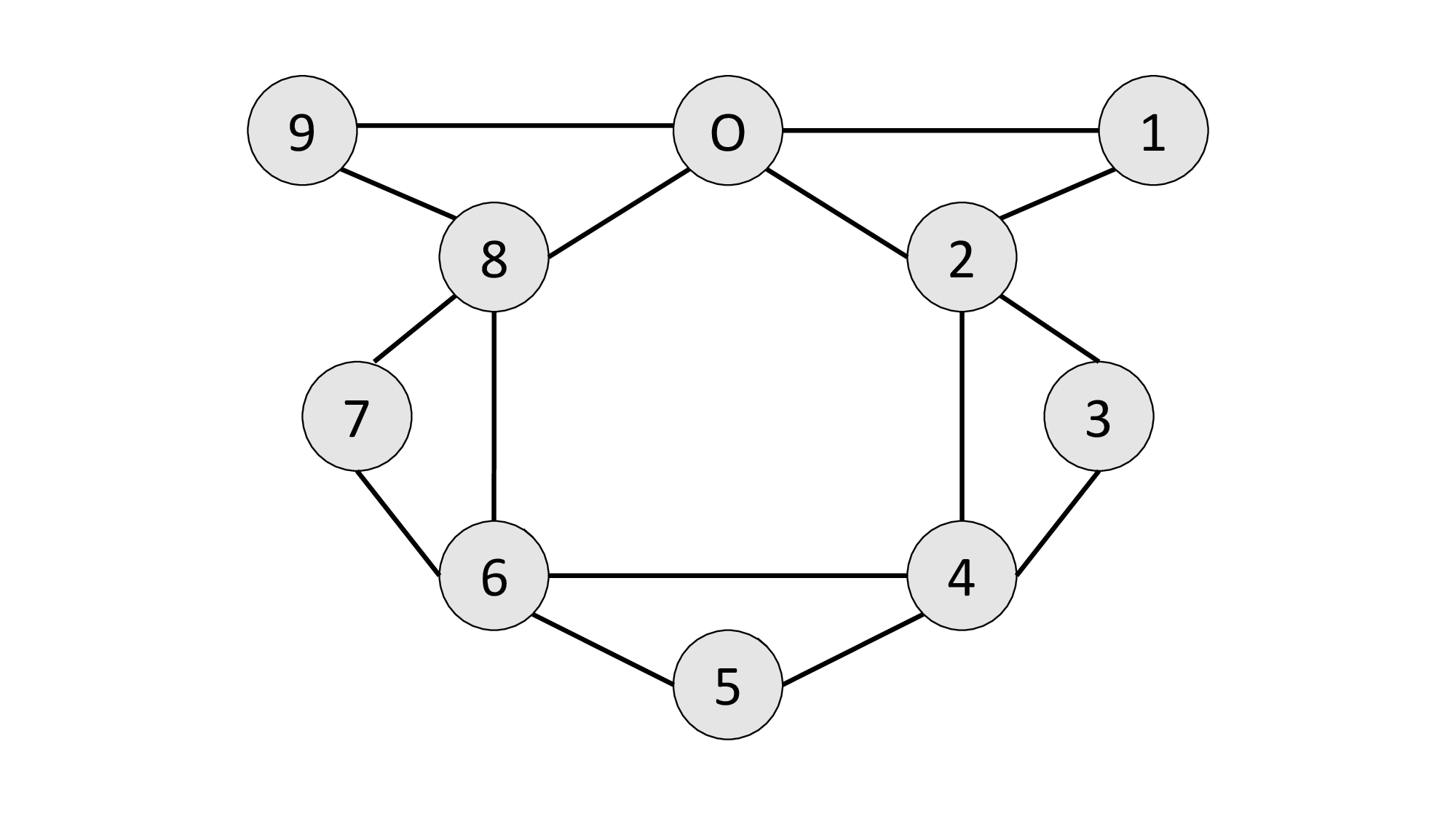}
	\end{minipage}
	\begin{minipage}{0.49\linewidth}
		\centering
		\begin{align*}\tiny
\setlength{\delimitershortfall}{0pt}
\begin{bmatrix}
    \frac{1}{5} & \frac{1}{5} & \frac{1}{5} & 0 & 0 & 0 & 0 & 0 & \frac{1}{5} &\frac{1}{5}\\[2ex] 
    \frac{1}{5} & \frac{3}{5} & \frac{1}{5} & 0 & 0 & 0 & 0 & 0 & 0 & 0\\[2ex] 
    \frac{1}{5} & \frac{1}{5} & \frac{1}{5} & \frac{1}{5} & \frac{1}{5} & 0 & 0 & 0 & 0 & 0 \\[2ex] 
    0 & 0 & \frac{1}{5} & \frac{3}{5} & \frac{1}{5} & 0 & 0 & 0 & 0 & 0\\[2ex] 
    0 & 0 & \frac{1}{5} & \frac{1}{5} & \frac{1}{5} & \frac{1}{5} & \frac{1}{5} & 0 & 0 & 0 \\[2ex] 
    0 & 0 & 0 & 0 & \frac{1}{5} & \frac{3}{5} & \frac{1}{5} & 0 & 0 & 0 \\[2ex]
    0 & 0 & 0 & 0 & \frac{1}{5} & \frac{1}{5} & \frac{1}{5} & \frac{1}{5} & \frac{1}{5} & 0\\[2ex]
    0 & 0 & 0 & 0 & 0 & 0 & \frac{1}{5} & \frac{3}{5} & \frac{1}{5} & 0\\[2ex]
    \frac{1}{5} & 0 & 0 & 0 & 0 & 0 & \frac{1}{5} & \frac{1}{5} & \frac{1}{5} & \frac{1}{5}\\[2ex]
    \frac{1}{5} & 0 & 0 & 0 & 0 & 0 & 0 & 0 & \frac{1}{5} & \frac{3}{5} 
\end{bmatrix}
  \end{align*}
	\end{minipage}%
 \caption{Left: The Ring-Like network topology. Circles denote the devices and edges denote connection links, where self-loops are omitted in the plot for brevity. Right: The corresponding mixing matrix $\bf W$.}
 \label{fig:ringtop}
\end{figure}

\newpage
\text{ }
\newpage

\end{document}